\documentclass[11pt,letterpaper,showkeys]{article} 
\usepackage{graphicx}
\usepackage{amssymb}
\usepackage{amsmath}
\usepackage{mathrsfs}
\usepackage{subfigure}
\newtheorem{theorem}{Theorem}[section] 
\newtheorem{lemma}[theorem]{Lemma} 
 
\newtheorem{proposition}[theorem]{Proposition} 
\newtheorem{remark}[theorem]{Remark}
\newtheorem{assumption}{Assumption} 

\newcommand{\beq}{\begin{equation}} 
\newcommand{\eeq}{\end{equation}} 
\newcommand{\beqa}{\begin{eqnarray}} 
\newcommand{\eeqa}{\end{eqnarray}} 
\newcommand{\beqas}{\begin{eqnarray*}} 
\newcommand{\eeqas}{\end{eqnarray*}} 
\newcommand{\ba}{\begin{array}} 
\newcommand{\ea}{\end{array}} 
\newcommand{\bi}{\begin{itemize}} 
\newcommand{\ei}{\end{itemize}} 
\newcommand{\gap}{\hspace*{2em}} 
 
\newcommand{\nn}{\nonumber}

\def\eqnok#1{(\ref{#1})}

\def\vgap{\vspace*{.1in}}
\def\QED{\ifhmode\unskip\nobreak\fi\ifmmode\ifinner\else\hskip5pt\fi\fi
  \hbox{\hskip5pt\vrule width5pt height5pt depth1.5pt\hskip1pt}}
\def\Arg{{\rm Arg}}  
\def\avg{{\rm avg}}  

\def\bI{{\bar I}} 
    
\def\bK{{\bar K}}
\def\bL{{\bar L}}
\def\bomega{{\bar \Omega}}

\def\bsigma{{\bf \Sigma}} 
\def\bsigmat{{\bf \Sigma^{t}}}

\def\cA{{\cal A}}
\def\cJ{{\cal J}}
\def\cN{{\cal N}}
\def\cO{{\cal O}}
\def\cP{{\cal P}}  
\def\cS{{\cal S}}
\def\cT{{\cal T}}  
 
\def\cX{{\cal X}} 
\def\cY{{\cal Y}}

\def\D{{\mathscr D}}

\def\eps{{\epsilon}}
\def\feas{{\rm feas}}

\def\sgn{{\rm sgn}}
\def\tD{{\widetilde{\D}}}

\def\tx{{\tilde x}}
\def\tX{{\tilde X}}
\def\ty{{\tilde y}}
\def\vrho{{\varrho}}
\def\vtheta{{\vartheta}}

\title{Sparse Approximation via Penalty Decomposition Methods
\thanks{This work was supported in part by NSERC Discovery Grant.}
} 

\author{
	Zhaosong Lu%
	\thanks{
	Department of Mathematics, Simon Fraser University, Burnaby, BC, 
	V5A 1S6, Canada. (email: {\tt zhaosong@sfu.ca}).} 
	\and
	Yong Zhang 
	\thanks{Department of Mathematics, 
	Simon Fraser University, Burnaby, BC, V5A 1S6, 
    Canada. (email: {\tt yza30@sfu.ca}).}
}

\date{February 19, 2012}

\begin{document}

\maketitle

\begin{abstract}

In this paper we consider sparse approximation problems, that is, general 
$l_0$ minimization problems with the $l_0$-``norm'' of a vector being a part of constraints 
or objective function. 
In particular, we first study the first-order optimality conditions for these problems. 
We then propose penalty decomposition (PD) methods for solving them in which
a sequence of penalty subproblems are solved by a block coordinate descent (BCD) method.  
Under some suitable assumptions, we establish that any accumulation point of the 
sequence generated by the PD methods satisfies the first-order optimality conditions 
of the problems. Furthermore, for the problems in which the $l_0$ part is the only 
nonconvex part, we show that such an accumulation point is a local minimizer of the 
problems. In addition, we show that any accumulation point of the sequence generated 
by the BCD method is a saddle point of the penalty subproblem. Moreover, for the problems 
in which the $l_0$ part is the only nonconvex part, we establish that such an accumulation point 
is a local minimizer of the penalty subproblem. Finally, we test the performance of our PD methods 
by applying them to sparse logistic regression, sparse inverse covariance selection, and 
compressed sensing problems. The computational results demonstrate that our methods generally 
outperform the existing methods in terms of solution quality and/or speed.         

\vskip14pt

 
\vskip14pt

 
\end{abstract}

\section{Introduction} \label{introduction}

Nowadays, there are numerous applications in which sparse solutions are concerned. For 
example, in compressed sensing, a large sparse signal is decoded by using a relatively 
small number of linear measurements, which can be formulated as finding a sparse solution 
to a system of linear equalities and/or inequalities. The similar ideas have also been 
widely used in linear regression. Recently, sparse inverse covariance selection 
becomes an important tool in discovering the conditional independence in graphical models. 
One popular approach for sparse inverse covariance selection is to find an approximate sparse 
inverse covariance while maximizing the log-likelihood (see, for example, \cite{De72}). Similarly, 
sparse logistic regression has been proposed as a promising method for feature selection in 
classification problems in which a sparse solution is sought to minimize the average logistic 
loss (see, for example, \cite{Ng04}). 
Mathematically, all these applications can be formulated into the following $l_0$ 
minimization problems:     
\beqa 
& \min\limits_{x \in \cX} \{f(x): g(x) \le 0, \ h(x)=0, \ \|x_J\|_0 \le r\},  \label{l0-J1} \\
& \min\limits_{x \in \cX} \{f(x) + \nu \|x_J\|_0: g(x) \le 0, \ h(x)=0\} \label{l0-J2}
\eeqa    
for some integer $r \ge 0$ and $\nu \ge 0$ controlling the sparsity of the solution, where 
$\cX$ is a closed convex set in the $n$-dimensional Euclidean space $\Re^n$, $f: \Re^n \to \Re$, 
$g: \Re^n \to \Re^m$ and $h:\Re^n \to \Re^p$ are continuously differentiable functions, and 
$\|x_J\|_0$ denotes the cardinality of the subvector formed by the entries of $x$ indexed by 
$J$. Some algorithms are proposed for solving special cases of these problems. For example, 
the iterative hard thresholding algorithms \cite{HeGiTr06,BlDa08,BlDa09} and matching pursuit 
algorithms \cite{MaZh93,Tr04} are developed for solving the $l_0$-regularized least squares 
problems arising in compressed sensing, but they cannot be applied 
to the general $l_0$ minimization problems \eqref{l0-J1} and \eqref{l0-J2}.  
In the literature, one popular approach for dealing with \eqref{l0-J1} and \eqref{l0-J2} 
is to replace $\|\cdot\|_0$ by the $l_1$-norm $\|\cdot\|_1$ and solve the resulting relaxation 
problems instead (see, for example, \cite{DaBaGh08,Ng04,ChDoSa98,Ti96}). For some 
applications such as compressed sensing, it has been shown in \cite{CaRoTa06} that 
under some suitable assumptions this approach is capable of solving \eqnok{l0-J1} and \eqnok{l0-J2}. 
Recently, another relaxation approach has been proposed to solve problems \eqnok{l0-J1} and 
\eqnok{l0-J2} in which $\|\cdot\|_0$ is replaced by $l_p$-``norm'' $\|\cdot\|_p$ for some 
$p \in (0,1)$ (see, for example, \cite{Ch07,ChXuYe09,ChZh10}). In general, it is not clear about 
the solution quality of these approaches. Indeed, for the example given in the Appendix, the 
$l_p$ relaxation approach for $p \in (0,1]$ fails to recover the sparse solution.        

In this paper we propose penalty decomposition (PD) methods for solving problems \eqnok{l0-J1} 
and \eqnok{l0-J2} in which a sequence of penalty subproblems are solved by a block coordinate descent 
(BCD) method. Under some suitable assumptions, we establish that any accumulation point of the 
sequence generated by the PD method satisfies the first-order optimality conditions of \eqnok{l0-J1} 
and \eqnok{l0-J2}. Furthermore,  when $h$'s are affine, and $f$ and $g$'s are convex, we show that 
such an accumulation point is a local minimizer of the problems. In addition, we show that any 
accumulation point of the sequence generated by the BCD method is a saddle point of the penalty 
subproblem. Moreover, when $h$'s are affine, and $f$ and $g$'s are convex, we establish that such 
an accumulation point is a local minimizer of the penalty subproblem. Finally, we test the performance 
of our PD methods by applying them to sparse logistic regression, sparse inverse covariance selection, 
and compressed sensing problems. The computational results demonstrate that our methods generally 
outperform the existing methods in terms of solution quality and/or speed.         

The rest of this paper is organized as follows. In Subsection \ref{notation},
we introduce the notation that is used throughout the paper. In Section \ref{sec:opt-conds}, 
we establish the first-order optimality conditions for general $l_0$ minimization problems. 
In Section \ref{tech}, we study a class of special $l_0$ minimization problems. We 
develop the PD methods for general $l_0$ minimization problems in Section \ref{method} 
and establish some convergence results for them. In Section \ref{results}, we conduct numerical 
experiments to test the performance of our PD methods for solving sparse logistic 
regression, sparse inverse covariance selection, and compressed sensing problems. 
Finally, we present some concluding remarks in section \ref{conclude}.    

\subsection{Notation} \label{notation}

In this paper, the symbols $\Re^n$ and $\Re^n_+$ denote the $n$-dimensional 
Euclidean space and the nonnegative orthant of $\Re^n$, respectively. Given a 
vector $v\in\Re^n$, the nonnegative part of $v$ is denoted by $v^+=\max(v,0)$, where 
the maximization operates entry-wise. For any real vector, $\|\cdot\|_0$ and $\|\cdot\|$ 
denote the cardinality (i.e., the number of nonzero entries) and the Euclidean norm of 
the vector, respectively. Given an index set $L \subseteq \{1,\ldots,n\}$, $|L|$ denotes 
the size of $L$, and the elements of $L$ are denoted by $L(1), \ldots, L(|L|)$, which are 
always arranged in ascending order. $x_L$ denotes the subvector formed by the entries of $x$ 
indexed by $L$. Likewise, $X_L$ denotes the submatrix formed by the columns of $X$ 
indexed by $L$. In addition,  For any two sets $A$ and $B$, the subtraction of $A$ 
and $B$ is given by $A\setminus B=\{x\in A: x\notin B\}$. Given a closed set 
$C \subseteq \Re^n$, let $\cN_C(x)$ and $\cT_C(x)$ denote the normal and tangent cones 
of $C$ at any $x \in C$, respectively. The space of all $m \times n$ matrices with real 
entries is denoted by $\Re^{m \times n}$, and the space of symmetric $n \times n$ matrices 
is be denoted by $\cS^n$. We denote by $I$ the identity matrix, whose 
dimension should be clear from the context. If $X \in \cS^n$ is positive semidefinite (resp., definite), 
we write $X \succeq 0$ (resp., $X \succ 0$). The cone of positive semidefinite (resp., definite) matrices is 
denoted by $\cS^n_+$ (resp., $\cS^n_{++}$). $\D$ is an operator which maps a vector to a 
diagonal matrix whose diagonal consists of the vector. Given an $n \times n$ matrix $X$,
 $\tD(X)$ denotes a diagonal matrix whose $i$th diagonal element is $X_{ii}$ for $i=1,\ldots,n$.
 
\section{First-order optimality conditions}
\label{sec:opt-conds}
 
 In this section we study the first-order optimality conditions for problems \eqnok{l0-J1} and 
\eqnok{l0-J2}. In particular, we first discuss the first-order necessary conditions for them. 
Then we study the first-order sufficient conditions for them when the $l_0$ part is the only 
nonconvex part.  

We now establish the first-order necessary optimality conditions for problems \eqnok{l0-J1} and 
\eqnok{l0-J2}.

\begin{theorem} \label{opt-cond-thm1}
Assume that $x^*$ is a local minimizer of problem \eqnok{l0-J1}. Let $J^*\subseteq J$ 
be an index set with $|J^*|=r$ such that $x^*_j = 0$ for all $j \in \bar J^*$, where 
$\bar J^*=J\setminus J^*$. Suppose that the following Robinson condition 
\beq \label{rob-cond-c1}
\left \{\left[\ba{c} 
g'(x^*)d - v \\ 
h'(x^*) d  \\
(I_{\bar J^*})^T d
\ea \right]: d \in \cT_{\cX}(x^*), v \in \Re^m, v_i \le 0, 
i \in \cA(x^*)\right\} = \Re^m \times \Re^p \times \Re^{|J|-r}
\eeq
holds, where $g'(x^*)$ and $h'(x^*)$ 
denote the Jacobian of the functions $g=(g_1,\ldots,g_m)$ and $h=(h_1,\ldots,h_p)$ 
at $x^*$, respectively, and
\beq \label{barJs}
 \cA(x^*) = \{1\le i \le m: g_i(x^*) = 0 \}.
\eeq
Then, there exists $(\lambda^*,\mu^*,z^*)\in \Re^m \times \Re^p \times \Re^n$ 
together with $x^*$ satisfying 
\beq \label{opt-cond-c}
\ba{c}
-\nabla f(x^*) -  \nabla g(x^*)\lambda^* - 
\nabla h(x^*)\mu^* - z^*  \in  \cN_{\cX}(x^*), \\ [5pt]
\lambda^*_i \ge 0, \ \lambda^*_i g_i(x^*) = 0,  \ i=1, \ldots, m; \ \ \ \ 
z^*_j = 0, \ j \in \bar J \cup J^*. 
\ea
\eeq
where $\bar J$ is the complement of $J$ in $\{1,\ldots,n\}$. 
\end{theorem}  

\begin{proof}
By the assumption that $x^*$ is a local minimizer of problem \eqnok{l0-J1}, one can observe that 
$x^*$ is also a local minimizer of the following problem:
\beq \label{NLP}
\min\limits_{x \in \cX} \{f(x): g(x) \le 0, \ h(x)=0, \ x_{\bar J^*}=0\}. 
\eeq
Using this observation, \eqnok{rob-cond-c1} and Theorem 3.25 of \cite{Rus06}, we see that 
the conclusion holds. 
\end{proof}

\vgap

\begin{theorem} \label{opt-cond-thm2}
Assume that $x^*$ is a local minimizer of problem \eqnok{l0-J2}. Let $J^* = 
\{j\in J: x^*_j \neq 0\}$ and $\bar J^*=J\setminus J^*$.
 suppose that the following Robinson condition 
\beq \label{rob-cond-c}
\left \{\left[\ba{c} 
g'(x^*)d - v \\ 
h'(x^*) d  \\
(I_{\bar J^*})^T d
\ea \right]: d \in \cT_{\cX}(x^*), v \in \Re^m, v_i \le 0, 
i \in \cA(x^*)\right\} = \Re^m \times \Re^p \times \Re^{|\bar J^*|}
\eeq
holds, where $\cA(x^*)$ is defined in \eqref{barJs}.
Then, there exists $(\lambda^*,\mu^*,z^*)\in \Re^m \times \Re^p \times \Re^n$ 
together with $x^*$ satisfying \eqref{opt-cond-c}. 
\end{theorem}

\begin{proof}
It is not hard to observe that $x^*$ is a local minimizer of problem \eqnok{l0-J2} if 
and only if $x^*$ is a local minimizer of problem \eqref{NLP}. Using this observation, 
\eqnok{rob-cond-c} and Theorem 3.25 of \cite{Rus06}, we see that the conclusion holds. 
\end{proof}

\vgap

We next establish the first-order sufficient optimality conditions for problems \eqnok{l0-J1} 
and \eqnok{l0-J2} when the $l_0$ part is the only nonconvex part.  

\begin{theorem} \label{suf-cond-thm1}
Assume that $h$'s are affine functions, and $f$ and $g$'s are convex functions. Let $x^*$ be a feasible point 
of problem \eqnok{l0-J1}, and let $\cJ^* = \{J^* \subseteq J: |J^*| = r, x^*_j=0, \forall j \in J\setminus  J^*\}$. 
Suppose that for any $J^* \in \cJ^*$,  there exists some $(\lambda^*,\mu^*,z^*)\in \Re^m \times \Re^p \times 
\Re^n$ such that \eqref{opt-cond-c} holds. Then, $x^*$ is a local minimizer of problem \eqnok{l0-J1}.
\end{theorem}

\begin{proof}
It follows from the above assumptions and Theorem 3.34 of \cite{Rus06} that $x^*$ is a minimizer 
of problem \eqref{NLP} for all $\bar J^* \in \{J\setminus J^*: J^* \in \cJ^*\}$. Hence, 
there exists $\epsilon >0$ such that $f(x) \ge f(x^*)$ for all $x \in \cup_{J^*\in\cJ^*}\cO_{J^*}(x^*;\epsilon)$, where 
\[
 \cO_{J^*}(x^*;\epsilon) = \{x\in\cX: g(x) \le 0, \ h(x)=0, \ x_{\bar J^*}=0, \ \|x-x^*\| < \epsilon\}
\]
with $\bar J^* = J\setminus J^*$. One can observe from \eqnok{l0-J1} that for any $x\in \cO(x^*;\epsilon)$, where
\[
\cO(x^*;\epsilon) = \{x\in\cX: g(x) \le 0, \ h(x)=0,\ \|x_J\|_0 \le r, \ \|x-x^*\| < \epsilon\},
\]
there exists $J^* \in \cJ^*$ such that $x \in \cO_{J^*}(x^*;\epsilon)$ and hence $f(x) \ge f(x^*)$. It implies 
that the conclusion holds.  
\end{proof}

\vgap

\begin{theorem} \label{suf-cond-thm2}
Assume that $h$'s are affine functions, and $f$ and $g$'s are convex functions. Let $x^*$ be a feasible point 
of problem \eqnok{l0-J2}, and let $J^* = \{j\in J: x^*_j \neq 0\}$. Suppose that for such $J^*$, there exists 
some $(\lambda^*,\mu^*,z^*)\in \Re^m \times \Re^p \times \Re^n$ such that \eqref{opt-cond-c} holds. Then, 
$x^*$ is a local minimizer of problem \eqnok{l0-J2}.
\end{theorem} 

\begin{proof}
By virtue of the above assumptions and Theorem 3.34 of \cite{Rus06}, we know that $x^*$ is a minimizer 
of problem \eqref{NLP} with $\bar J^* = J\setminus J^*$. Also, we observe that any point is a local minimizer 
of problem \eqnok{l0-J2} if and only if it is a local minimizer of problem \eqref{NLP}. It then implies that 
$x^*$ is a local minimizer of \eqnok{l0-J2}.  
\end{proof}

\vgap

\begin{remark}
The second-order necessary or sufficient optimality conditions for problems \eqnok{l0-J1} and \eqnok{l0-J2} 
can be similarly established as above. 
\end{remark}

\section{A class of special $l_0$ minimization}
\label{tech}

In this section we show that a class of special $l_0$ minimization problems have closed-form solutions, 
which can be used to develop penalty decomposition methods for solving general $l_0$ minimization problems.

\begin{proposition} \label{prop1}
Let $\cX_i \subseteq \Re$ and $\phi_i: \Re \to \Re$ for $i=1,\ldots,n$ be given. Suppose 
that $r$ is a positive integer and $0 \in \cX_i$ for all $i$. Consider the following 
$l_0$ minimization problem:
\beq \label{l0-p1}
\min\left\{\phi(x) = \sum^n_{i=1} \phi_i(x_i): \|x\|_0 \le r, \ x \in \cX_1 \times 
\cdots \times \cX_n \right\}. 
\eeq
Let $\tx^*_i\in \Arg\min\{\phi_i(x_i): x_i \in \cX_i\}$ and $I^* \subseteq \{1,\ldots, n\}$ be 
the index set corresponding to $r$ largest values of $\{v^*_i\}^n_{i=1}$, where 
$v^*_i = \phi_i(0)-\phi_i(\tx^*_i)$ for $i=1, \ldots, n$. Then, $x^*$ is an optimal solution of 
problem \eqnok{l0-p1}, where $x^*$ is defined as follows:
\[
x^*_i = \left\{\ba{ll}
\tx^*_i & \mbox{if} \ i \in I^*; \\
0  & \mbox{otherwise}, 
\ea\right. \quad i=1, \ldots, n.
\] 
\end{proposition}

\begin{proof}
By the assumption that $0 \in \cX_i$ for all $i$, and the definitions of $x^*$, $\tx^*$ and $I^*$, we 
see that $x^*\in \cX_1 \times \cdots \times \cX_n$ and $\|x^*\|_0 \le r$. Hence, $x^*$ is 
a feasible solution of \eqnok{l0-p1}. It remains to show that $\phi(x) \ge \phi(x^*)$ for any feasible 
point $x$ of \eqnok{l0-p1}. Indeed, let $x$  be arbitrarily chosen such that $\|x\|_0 \le r$ and 
$x \in \cX_1 \times \cdots \times \cX_n$, and let $L = \{i: x_i \neq 0\}$. Clearly, $|L| \le r = |I^*|$.
Let $\bI^*$ and $\bL$ denote the complement of $I^*$ and $L$ in $\{1,\ldots,n\}$, respectively. It then 
follows that 
\[
|\bL \cap I^*| \ = \ |I^*|-|I^*\cap L| \ \ge \ |L| -|I^*\cap L| \ = \ |L \cap \bI^*|.
\]
In view of the definitions of $x^*$, $\tx^*$, $I^*$, $\bI^*$, $L$ and $\bL$, we further have 
\[
\ba{lcl}
\phi(x) - \phi(x^*) &=& \sum_{i\in L \cap I^*} (\phi_i(x_i) - \phi_i(x^*_i)) + 
\sum_{i\in \bL \cap \bI^*}(\phi_i(x_i) - \phi_i(x^*_i))
 \\ [4pt]
& & + \sum_{i\in \bL \cap I^*}(\phi_i(x_i) - \phi_i(x^*_i)) + 
\sum_{i\in L \cap \bI^*} (\phi_i(x_i) - \phi_i(x^*_i)), \\ [5pt]
&=& \sum_{i\in L \cap I^*} (\phi_i(x_i) - \phi_i(\tx^*_i)) + 
\sum_{i\in \bL \cap \bI^*}(\phi_i(0) - \phi_i(0))
 \\ [4pt]
& & + \sum_{i\in \bL \cap I^*}(\phi_i(0) - \phi_i(\tx^*_i)) + 
\sum_{i\in L \cap \bI^*} (\phi_i(x_i) - \phi_i(0)), \\ [5pt]
& \ge & \sum_{i\in \bL \cap I^*}(\phi_i(0) - \phi_i(\tx^*_i))
+ \sum_{i\in L \cap \bI^*} (\phi_i(\tx^*_i) - \phi_i(0)), \\ [5pt]
& = & \sum_{i\in \bL \cap I^*}(\phi_i(0) - \phi_i(\tx^*_i))
- \sum_{i\in L \cap \bI^*} (\phi_i(0) - \phi_i(\tx^*_i)) \ \ge \ 0, 
\ea 
\]
where the last inequality follows from the definition of $I^*$ and the relation $|\bL \cap I^*| 
\ge |L \cap \bI^*|$. Thus, we see that $\phi(x) \ge \phi(x^*)$ for any feasible 
point $x$ of \eqnok{l0-p1}, which implies that the conclusion holds. 
\end{proof}

\vgap

It is straightforward to establish the following result.

\begin{proposition} \label{prop2}
Let $\cX_i \subseteq \Re$ and $\phi_i: \Re \to \Re$ for $i=1,\ldots,n$ be given. Suppose 
that $\nu \ge 0$ and $0 \in \cX_i$ for all $i$. Consider the following $l_0$ minimization 
problem:
\beq \label{l0-p2}
\min\left\{\nu \|x\|_0 + \sum^n_{i=1} \phi_i(x_i): \ \ x \in \cX_1 \times \cdots \times \cX_n \right\}. 
\eeq
Let $\tx^*_i\in \Arg\min\{\phi_i(x_i): x_i \in \cX_i\}$ and $v^*_i = \phi_i(0)-\nu-\phi_i(\tx^*_i)$ for 
$i=1, \ldots, n$. Then, $x^*$ is an optimal solution of problem \eqnok{l0-p2}, where $x^*$ is defined 
as follows:
\[
x^*_i = \left\{\ba{ll}
\tx^*_i & \mbox{if} \ v^*_i \ge 0; \\
0  & \mbox{otherwise}, 
\ea\right. \quad i=1, \ldots, n.
\] 
\end{proposition}

\section{Penalty decomposition methods for general $l_0$ minimization}
\label{method} 

In this section we propose penalty decomposition (PD) methods for solving general $l_0$ minimization 
problems \eqnok{l0-J1} and \eqnok{l0-J2} and establish their convergence. Throughout this section, we 
make the following assumption for problems \eqnok{l0-J1} and \eqnok{l0-J2}.

\begin{assumption} \label{assump}
Problems \eqnok{l0-J1} and \eqnok{l0-J2} are feasible, and moreover, at least a feasible 
solution, denoted by  $x^{\rm feas}$, is known.  
\end{assumption} 

This assumption will be used to design the PD methods with nice convergence properties. It can 
be dropped, but the theoretical convergence of the corresponding PD methods may become weaker. 
We shall also mention that, for numerous real applications, $x^{\rm feas}$ is readily available 
or can be observed from the physical background of problems. For example, all application 
problems discussed in Section \ref{results} have a trivial feasible solution. On the other 
hand, for some problems which do not have a trivial feasible solution, one can always approximate 
them by the problems which have a trivial feasible solution. For instance, problem 
\eqnok{l0-J1} can be approximately solved as the following problem:
\[
\min\limits_{x \in \cX} \{f(x)+\rho (\|u^+\|^2+\|v\|^2): g(x)-u \le 0, \ h(x)-v=0, \ \|x_J\|_0 \le r\} 
\]
for some large $\rho$. The latter problem has a trivial feasible solution when $\cX$ is 
sufficiently simple.    

\subsection{Penalty decomposition method for problem \eqnok{l0-J1}} 
\label{sec:PD-c} 


In this subsection we propose a PD method for solving problem \eqnok{l0-J1} and establish its 
convergence.

We observe that \eqnok{l0-J1} can be equivalently reformulated as 
\beq \label{l0-J1-ref}
\min\limits_{x \in \cX, y\in \cY} \{f(x): \ g(x) \le 0, \ h(x)=0, \ x_J-y=0\},   
\eeq
where 
\[
\cY = \{y\in \Re^{|J|}: \|y\|_0 \le r\}.
\]
The associated quadratic penalty function is defined as follows: 
\beq \label{q-fun}
q_\vrho(x,y) = f(x) + \frac{\vrho}{2}(\|[g(x)]^+\|^2+\|h(x)\|^2+\|x_J-y\|^2)  \quad \forall x \in \cX, y\in \cY
\eeq
for some penalty parameter $\vrho>0$. 

We are now ready to propose a PD method for solving problem \eqnok{l0-J1-ref} (or equivalently, \eqref{l0-J1}) 
in which each penalty subproblem is approximately solved by a block coordinate descent (BCD) method.

\gap

\noindent
{\bf Penalty decomposition method for \eqnok{l0-J1}:}  \\ [5pt]
Let $\{\epsilon_k\}$ be a positive decreasing sequence. Let $\vrho_0 >0$, 
$\sigma > 1$ be given. Choose an arbitrary $y^0_0\in \cY$ and a constant 
$\Upsilon \ge \max\{f(x^{\feas}), \min_{x\in \cX} q_{\vrho_0}(x,y^0_0)\}$. 
Set $k=0$. 
\begin{itemize}
\item[1)] Set $l=0$ and apply the BCD method to find an approximate solution 
$(x^k, y^k) \in \cX \times \cY$ for the penalty subproblem 
\beq \label{inner-prob1}
\min\{q_{\vrho_k}(x,y): \ x \in \cX, \ y \in \cY\} 
\eeq
by performing steps 1a)-1d): 
\bi
\item[1a)] Solve $x^{k}_{l+1} \in \Arg\min\limits_{x \in \cX} q_{\vrho_k}(x,y^k_{l})$.
\item[1b)]
Solve $y^k_{l+1} \in \Arg\min\limits_{y \in \cY} q_{\vrho_k}(x^k_{l+1},y)$.
\item[1c)] Set $(x^k, y^k) := (x^k_{l+1},y^k_{l+1})$. 
If $(x^k,y^k)$ satisfies 
\beq \label{inner-cond-c}
\|\cP_\cX(x^k-\nabla_x q_{\vrho_k}(x^k,y^k))-x^k\|  \le  \epsilon_k, 
\eeq
then go to step 2). 
\item[1d)] Set $l \leftarrow l+1$ and go to step 1a). 
\ei
\item[2)]
Set $\vrho_{k+1} := \sigma\vrho_k$.
\item[3)] 
If $\min\limits_{x \in \cX} q_{\vrho_{k+1}}(x,y^k) > \Upsilon$, set 
$y^{k+1}_0 := x^{\feas}$. Otherwise, set $y^{k+1}_0 := y^k$.
\item[4)]
Set $k \leftarrow k+1$ and go to step 1). 
\end{itemize}
\noindent
{\bf end}

\vgap

\begin{remark}
The condition \eqnok{inner-cond-c} will be used to establish the global convergence 
of the above method. It may not be easily verifiable unless $\cX$ is simple. On 
the other hand, we observe that the sequence $\{q_{\vrho_k}(x^k_l,y^k_l)\}$ is 
non-increasing for any fixed $k$. In practice, it is thus reasonable to terminate 
the BCD method based on the progress of $\{q_{\vrho_k}(x^k_l,y^k_l)\}$. Another 
practical termination criterion for the BCD method is based on the relative 
change of the sequence $\{(x^k_l,y^k_l)\}$, that is, 
\beq \label{inner-term}
\max\left\{\frac{\|x^k_l - x^{k}_{l-1}\|_\infty}
           {\max(\|x^k_l\|_\infty,1)}, \frac{\|y^k_l - y^{k}_{l-1}\|_\infty}
           {\max(\|y^k_l\|_\infty,1)} \right\} \ \leq \ \eps_I 
\eeq
for some $\eps_I >0$. In addition, we can terminate the outer iterations of the PD 
method once
\beq \label{outer-term}
\|x^k-y^k\|_\infty \ \leq \ \eps_O
\eeq
for some $\eps_O >0$. Given that problem \eqnok{inner-prob1} is nonconvex, the BCD 
method may converge to a stationary point. To enhance the performance of the BCD method, 
one may execute it multiple times by restarting from a suitable perturbation of the 
current best approximate solution. For example, at the $k$th outer iteration, let 
$(x^k,y^k)$ be the current best approximate solution of \eqnok{inner-prob1} found 
by the BCD method, and let $r_k = \|y^k\|_0$. Assume that $r_k > 1$. Before starting 
the $(k+1)$th outer iteration, one can re-apply the BCD method 
starting from $y^k_0 \in \Arg\min\{\|y-y^k\|: \|y\|_0 \le r_k - 1\}$ and obtain a new 
approximate solution $(\tilde x^k, \tilde y^k)$ of \eqnok{inner-prob1}. If $q_{\vrho_k}(\tx^k,\ty^k)$ 
is ``sufficiently'' smaller than $q_{\vrho_k}(x^k,y^k)$, one can set $(x^k,y^k) := (\tx^k,\ty^k)$ and 
repeat the above process. Otherwise, one can terminate the $k$th outer iteration and start the next 
outer iteration. Finally, it follows from Proposition \ref{prop1} that the subproblem in step 1b) 
has a closed-form solution. 
\end{remark}

\vgap

We next establish a convergence result regarding the inner iterations of the above 
PD method. In particular, we will show that an approximate solution 
$(x^k,y^k)$ of problem \eqref{inner-prob1} satisfying \eqref{inner-cond-c} can be found by 
the BCD method described in steps 1a)-1d). For notational convenience, we omit 
the index $k$ from \eqnok{inner-prob1} and consider the BCD method for solving the 
problem 
\beq \label{pd-subprob-c}
\min \{q_\vrho(x,y): x\in \cX, \ y \in \cY\}
\eeq
instead. Accordingly, we rename the iterates of the above BCD method and present it as 
follows.

\gap

\noindent
{\bf Block coordinate descent method for \eqnok{pd-subprob-c}:} \\ [5pt]
Choose an arbitrary initial point $y^0 \in \cY$. Set $l=0$. 
\begin{itemize}
\item[1)]
Solve $x^{l+1} \in \Arg\min\limits_{x \in \cX} q_\vrho(x,y^l)$.
\item[2)]
Solve $y^{l+1} \in \Arg\min\limits_{y \in \cY} q_\vrho(x^{l+1},y)$.
\item[3)]
Set $l \leftarrow l+1$ and go to step 1). 
\end{itemize}
\noindent
{\bf end}

\vgap
\begin{lemma} \label{local-soln-c}
Suppose that $(x^*,y^*)\in \Re^n \times \Re^{|J|}$ is a saddle point of problem 
\eqref{pd-subprob-c}, that is,   
\beq \label{saddle-pt-c}
x^* \in \Arg\min\limits_{x \in \cX} q_{\vrho}(x,y^*),  \ \ \ 
y^* \in \Arg\min\limits_{y \in \cY} q_{\vrho}(x^*,y).
\eeq 
Furthermore, assume that $h$'s are affine functions, and $f$ and $g$'s 
are convex functions. Then, $(x^*,y^*)$ is a local minimizer of 
problem \eqref{pd-subprob-c}.
\end{lemma}

\begin{proof}
Let $K=\{i: y^*_i \neq 0\}$, and let $h_x$, $h_y$ be any two vectors such that 
$x^*+h_x \in \cX$, $y^*+h_y \in \cY$ and $|(h_y)_i| < |y^*_i|$ for all $i \in K$. 
Claim that 
\beq \label{yeq}
(y^*-x^*_J)^T h_y = 0.
\eeq
If $\|x^*_J\|_0 >r$, we observe from the second relation of \eqref{saddle-pt-c} 
and Proposition \ref{prop1} that $\|y^*\|_0=r$ and $y^*_i = x^*_{J(i)}$ for 
all $i\in K$, which, together with $y^*+h_y \in \cY$ and $|(h_y)_i| < |y^*_i|$ 
for all $i \in K$, implies that $(h_y)_i = 0$ for all $i \notin K$ and hence 
\eqref{yeq} holds. On the other hand, if $\|x^*_J\|_0 \le r$,  one can observe 
that $y^*=x^*_J$ and thus \eqref{yeq} also holds. In addition, by the assumption 
that $h$'s are affine functions, and $f$ and $g$'s are convex functions, we know that 
$q_\vrho$ is convex. It then follows from the first relation of \eqref{saddle-pt-c} 
and the first-order optimality condition that $[\nabla_x q_\vrho(x^*,y^*)]^T h_x \ge 0$. 
Using this relation along with \eqref{yeq} and the convexity of $q_\vrho$, we have
\[
\ba{lcl}
q_\vrho(x^*+h_x,y^*+h_y) & \ge & q_\vrho(x^*,y^*) + [\nabla_x q_\vrho(x^*,y^*)]^T h_x 
+[\nabla_y q_\vrho(x^*,y^*)]^T h_y \\ [4pt]
&=& q_\vrho(x^*,y^*) + [\nabla_x q_\vrho(x^*,y^*)]^T h_x 
+\vrho(y^*-x^*_J)^T h_y \ \ge \ q_\vrho(x^*,y^*), 
\ea
\]
which together with the above choice of $h_x$ and $h_y$ implies that $(x^*,y^*)$ 
is a local minimizer of \eqref{pd-subprob-c}.
\end{proof}

\vgap

\begin{theorem} \label{BCD-thm-c}
Let $\{(x^l,y^l)\}$ be the sequence generated by the above BCD method, and let 
$\epsilon >0$ be given. Suppose that $(x^*,y^*)$ is an accumulation point of 
$\{(x^l,y^l)\}$. Then the following statements hold:
\bi
\item[(a)] $(x^*,y^*)$ is a saddle point of problem \eqref{pd-subprob-c}.
\item[(b)] There exists some $l>0$ such that 
\[
\|\cP_{\cX}(x^l-\nabla_x q_\vrho(x^l,y^l))-x^l\| < \epsilon.
\]
\item[(c)] Furthermore, if $h$'s are affine functions, and $f$ and $g$'s 
are convex functions, then $(x^*,y^*)$ is a local minimizer of problem 
\eqref{pd-subprob-c}.
\ei 
\end{theorem}

\begin{proof} 
We first show that statement (a) holds. Indeed, one can observe that 
\beqa
q_\vrho(x^{l+1},y^l) & \le & q_\vrho(x,y^l) \ \ \ \forall x \in \cX, \label{x-update-c} \\
q_\vrho(x^l,y^l) & \le & q_\vrho(x^l,y) \ \ \ \forall y \in \cY. \label{y-update-c}
\eeqa
It follows that 
\beq \label{squeeze-c}
q_\vrho(x^{l+1},y^{l+1}) \ \le \ q_\vrho(x^{l+1},y^l) \ \le \ q_\vrho(x^l,y^l) \ \ \ \forall l \ge 1.
\eeq
Hence, the sequence $\{q_\vrho(x^{l},y^{l})\}$ is non-increasing. Since $(x^*,y^*)$ is an accumulation 
point of $\{(x^l,y^l)\}$, there exists a subsequence $L$ such that $\lim_{l\in L \to \infty} (x^l,y^l) = (x^*,y^*)$. 
We then observe that $\{q_\vrho(x^{l},y^{l})\}_{l\in L}$ is bounded, which together with the 
monotonicity of $\{q_\vrho(x^{l},y^{l})\}$ implies that $\{q_\vrho(x^{l},y^{l})\}$ is bounded below and hence 
$\lim_{l\to\infty}q_\vrho(x^{l},y^{l})$ exists. This observation, \eqref{squeeze-c} and the continuity of $q_\vrho(\cdot,\cdot)$ yield 
\[
\lim_{l\to\infty} q_\vrho(x^{l+1},y^l) = \lim_{l\to\infty} q_\vrho(x^{l},y^{l}) = \lim_{l\in L\to\infty} 
q_\vrho(x^{l},y^{l}) = q_\vrho(x^*,y^*).
\]
Using these relations, the continuity of $q_\vrho(\cdot,\cdot)$, and taking limits on both sides of 
\eqref{x-update-c} and \eqref{y-update-c} as $l \in L \to \infty$, we have
\beqa
q_\vrho(x^*,y^*) & \le & q_\vrho(x,y^*) \ \ \ \forall x \in \cX, \label{sad-pt-c1}\\
q_\vrho(x^*,y^*) & \le & q_\vrho(x^*,y) \ \ \ \forall y \in \cY. \label{sad-pt-c2}
\eeqa
In addition, from the definition of $\cY$, we know that $\|y^l\|_0 \le r$, which immediately implies 
$\|y^*\|_0 \le r$. Also, $x^* \in \cX$ due to the closedness of $\cX$. This together with \eqref{sad-pt-c1} and 
\eqref{sad-pt-c2} implies that $(x^*,y^*)$ is a saddle point of \eqref{pd-subprob-c} and hence 
statement (a) holds. Using \eqref{sad-pt-c1} and 
the first-order optimality condition, we have 
\[
\|\cP_{\cX}(x^*-\nabla_x q_\vrho(x^*,y^*))-x^*\| = 0.
\]
By the continuity of $\cP_\cX(\cdot)$ and $\nabla_x q_\vrho(\cdot,\cdot)$, and the relation 
$\lim_{l\in L \to \infty} (x^l,y^l) = (x^*,y^*)$, one can see that 
\[
\lim_{l \in L \to \infty} \|\cP_{\cX}(x^l-\nabla_x q_\vrho(x^l,y^l))-x^l\|=0,
\] 
and hence, statement (b) immediately follows. In addition, statement (c) holds due to statement (a) 
and Lemma \ref{local-soln-c}.  
\end{proof}

\vgap

The following theorem establishes the convergence of the outer iterations of the PD method 
for solving problem \eqref{l0-J1}. In particular, we show that under some suitable assumption, 
any accumulation point of the sequence generated by the PD method satisfies the first-order 
optimality conditions of \eqref{l0-J1}. Moreover, when the $l_0$ part is the only nonconvex 
part, we show that under some assumption, the accumulation point is a local minimizer of 
\eqref{l0-J1}. 

\begin{theorem} \label{main-thm-c}
Assume that $\epsilon_k \to 0$. Let $\{(x^k,y^k)\}$ be the sequence generated by the 
above PD method, $I_k = \{i^k_1,\ldots,i^k_r\}$ be a 
set of $r$ distinct indices in $\{1,\ldots,|J|\}$  
such that $(y^k)_i = 0$ for any $i \notin I_k$, and let $J_k = \{J(i): i\in I_k\}$.
Suppose that the level set 
$\cX_{\Upsilon} := \{x \in \cX: f(x) \le \Upsilon \}$ is compact. Then, the following 
statements hold:
\bi
\item[(a)] The sequence $\{(x^k,y^k)\}$ is bounded.
\item[(b)] Suppose $(x^*,y^*)$ is an accumulation point of $\{(x^k,y^k)\}$. Then, 
$x^*=y^*$ and $x^*$ is a feasible point of problem \eqnok{l0-J1}. Moreover, there exists a subsequence 
$K$ such that $\{(x^k,y^k)\}_{k \in K} \to (x^*,y^*)$, $I_k = I^*$ and $J_k = J^*:= \{J(i): i\in I^*\}$ 
for some index set $I^* \subseteq \{1,\ldots,|J|\}$ when $k \in K$ is sufficiently large.
\item[(c)] Let $x^*$, $K$ and $J^*$ be defined above, and let $\bar J^*=J\setminus J^*$. 
Suppose that the Robinson condition \eqref{rob-cond-c1} holds at $x^*$ for such $\bar J^*$.
Then, $\{(\lambda^{k},\mu^{k},\varpi^k)\}_{k\in K}$
is bounded, where 
\beq \label{lag-multip1}
\lambda^{k} = \vrho_k [g(x^k)]^+, \ \ \ \
\mu^{k} = \vrho_k h(x^k), \ \ \ \ \varpi^k = \vrho_k(x^k_{J}-y^k).
\eeq
Moreover, each accumulation point $(\lambda^*,\mu^*,\varpi^*)$ of 
$\{(\lambda^k,\mu^k,\varpi^k)\}_{k\in K}$ together with $x^*$ 
satisfies the first-order optimality conditions \eqref{opt-cond-c} 
with $z^*_j=\varpi^*_i$ for all $j=J(i) \in \bar J^*$. Further, if 
$\|x^*_J\|_0=r$, $h$'s are affine functions, and $f$ and $g$'s are convex 
functions, then $x^*$ is a local minimizer of problem \eqnok{l0-J1}. 
\ei
\end{theorem} 
 
\begin{proof}
In view of \eqnok{q-fun} and our choice of $y^k_0$ that is specified in step 3), 
one can observe that 
\beq \label{bound-f}
f(x^k) + \frac{\vrho_k}{2}(\|[g(x^k)]^+\|^2+\|h(x^k)\|^2+\|x^k_J-y^k\|^2)
= q_{\vrho_k}(x^k,y^k)  \le 
\min\limits_{x \in \cX} q_{\vrho_{k}}(x,y^k_0) \le \Upsilon \quad \forall k.
\eeq
It immediately implies that $\{x^k\} \subseteq \cX_\Upsilon$, and hence, $\{x^k\}$ is 
bounded. Moreover, we can obtain from \eqnok{bound-f} that  
\[
\|x^k_J-y^k\|^2 \le 2[\Upsilon -f(x^k)]/\vrho_k \le 2[\Upsilon -\min_{x\in\cX_\Upsilon}f(x)]/\vrho_0,
\]
which together with the boundedness of  $\{x^k\}$ yields that $\{y^k\}$ is bounded. Therefore, 
statement (a) follows. We next show that statement (b) also holds. Since $(x^*,y^*)$ 
is an accumulation point of $\{(x^k,y^k)\}$, there exists a subsequence $\{(x^k,y^k)\}_{k\in \bK} 
\to (x^*,y^*)$. Recall that $I_k$ is an index set. It follows that $\{(i^k_1, \ldots,i^k_r)\}_{k\in \bK}$ 
is bounded for all $k$. Thus there exists a subsequence $K \subseteq \bK$ such that 
$\{(i^k_1,\ldots,i^k_r)\}_{k \in K} \to (i^*_1,\ldots,i^*_r)$ for some $r$ distinct indices 
$i^*_1, \ldots,  i^*_r$. Since $i^k_1, \ldots, i^k_r$ are $r$ distinct integers, one can easily 
conclude that $(i^k_1,\ldots,i^k_r) = (i^*_1,\ldots,i^*_r)$ for 
sufficiently large $k\in K$. Let $I^* = \{i^*_1,\ldots,i^*_r\}$. It then follows that 
$I_k = I^*$ and $J_k = J^*$ when $k\in K$ is sufficiently large, and moreover, 
$\{(x^k,y^k)\}_{k\in K} \to (x^*,y^*)$. Therefore, statement (b) holds. Finally, we show that 
statement (c) holds. Indeed, let $s^k$ be the vector such that  
\[
\cP_\cX(x^k-\nabla_x q_{\vrho_k}(x^k,y^k)) = x^k + s^k.
\] 
It then follows from \eqref{inner-cond-c} that $\|s^k\| \le \epsilon_k$ 
for all $k$, which together with $\lim_{k\to\infty}\epsilon_k=0$ implies 
$\lim_{k \to \infty} s^k =0$.
By a well-known property of the projection map $\cP_\cX$, we have
\[
(x-x^k-s^k)^T[x^k-\nabla_x q_{\vrho_k}(x^k,y^k)-x^k-s^k] \ \le \ 0, \ \forall x \in \cX.
\]   
Hence, we obtain that
\beq \label{inner-opt-cond}
-\nabla_x q_{\vrho_k}(x^k,y^k) - s^k \in \cN_\cX(x^k+s^k).
\eeq
Using this relation, \eqref{inner-opt-cond}, \eqref{lag-multip1} and the definition of $q_\vrho$, 
we have
\beq \label{inner-term-1}
-\nabla f(x^k) - \nabla g(x^k) \lambda^k - 
\nabla h(x^k) \mu^k - I_{J} \varpi^k - s^k \in \cN_\cX(x^k+s^k).
\eeq
We now claim that $\{(\lambda^k,\mu^k,\varpi^k)\}_{k\in K}$ is bounded. 
Suppose for contradiction that it is unbounded. By passing to a subsequence if 
necessary, we can assume that 
$\{\|(\lambda^{k},\mu^{k},\varpi^k)\|\}_{k\in K} \to \infty$. Let 
$(\bar \lambda^k, \bar \mu^k, \bar \varpi^k) = (\lambda^k,\mu^k,\varpi^k) /{\|(\lambda^k, \mu^k, \varpi^k)\|}$. 
Without loss of generality, we assume that $\{(\bar \lambda^{k}, \bar \mu^{k},\bar\varpi^k)\}_{k\in K} 
\to (\bar\lambda, \bar\mu,\bar\varpi)$ (otherwise, one can consider its convergent subsequence). 
Clearly, $\|(\bar\lambda,\bar\mu,\bar\varpi)\|=1$. Dividing both sides of \eqnok{inner-term-1} 
by $\|(\lambda^k, \mu^k, \varpi^k)\|$, taking limits as $k\in K \to \infty$, and using the 
relation $\lim_{k\in K\to\infty}s^k=0$ and the semicontinuity of $\cN_\cX(\cdot)$,  we obtain that 
\beq \label{lim-eq}
-\nabla g(x^*)\bar\lambda - \nabla h(x^*)\bar\mu - I_J \bar\varpi  \in \cN_\cX(x^*). 
\eeq
We can see from \eqref{barJs} and \eqref{lag-multip1} that 
$\bar\lambda\in\Re^m_+$, and $\bar\lambda_i=0$ for $i\notin\cA(x^*)$.      
Also, from Proposition \ref{prop1} and the definitions of $y_k$, $I_k$ and 
$J_k$, one can observe that $x^k_{J_k} = y^k_{I_k}$ and hence $\varpi^k_{I_k}=0$. 
In addition, we know from statement (b) that $I_k = I^*$ when $k \in K$ is sufficiently large. 
Hence, $\bar\varpi_{I^*}=0$.
Since Robinson's condition \eqnok{rob-cond-c1} 
is satisfied at $x^*$, there exist $d\in \cT_\cX(x^*)$ and $v \in \Re^m$ such that $v_i \le 0$ 
for $i\in\cA(x^*)$, and 
\[
g'(x^*)d - v = -\bar\lambda, \ \ \ \ \ h'(x^*) d = -\bar\mu, \ \ \ \ \ (I_{\bar J^*})^T d = -\bar\varpi_{\bar I^*},
\]
where $\bar I^*$ is the complement of $I^*$ in $\{1,\ldots, |J|\}$. Recall that $\bar\lambda\in\Re^m_+$, 
$\bar\lambda_i=0$ for $i\notin \cA(x^*)$, and $v_i \le 0$ for $i\in\cA(x^*)$. Hence, $v^T\bar\lambda \le 0$. 
In addition, since $\bar\varpi_{I^*}=0$, one has $I_J\bar\varpi =I_{\bar J^*}\bar\varpi_{\bar I^*}$. 
Using these relations, \eqref{lim-eq}, and the facts that $d\in \cT_\cX(x^*)$ and $\bar\varpi_{I^*}=0$, we have 
\[
\ba{lcl}
\|\bar\lambda\|^2 + \|\bar\mu\|^2 + \|\bar\varpi\|^2 &=& 
-[(-\bar\lambda)^T\bar\lambda + (-\bar\mu)^T \bar\mu + (-\bar\varpi_{\bar I^*})^T\bar\varpi_{\bar I^*}] \\ [5pt]
&=& - [(g'(x^*)d - v)^T\bar\lambda + (h'(x^*) d)^T \bar\mu + ((I_{\bar J^*})^T d)^T\bar\varpi_{\bar I^*}] \\ [5pt]
&=& d^T(-\nabla g(x^*)\bar\lambda - \nabla h(x^*)\bar\mu - I_J \bar\varpi) + v^T \bar\lambda  \ \le \ 0.
\ea  
\]
It yields $(\bar\lambda,\bar\mu,\bar\varpi)=(0,0,0)$, which contradicts the identity 
$\|(\bar\lambda,\bar\mu,\bar\varpi)\|=1$. Therefore, the subsequence 
$\{(\lambda^{k},\mu^{k},\varpi^k)\}_{k\in K}$ is bounded. Let $(\lambda^*,\mu^*,\varpi^*)$ be 
an accumulation point of $\{(\lambda^k,\mu^k,\varpi^k)\}_{k\in K}$. By passing to a subsequence if 
necessary, we can assume that $(\lambda^k,\mu^k,\varpi^k) \to (\lambda^*,\mu^*,\varpi^*)$ as 
$k\in K \to \infty$. Taking limits on both sides of \eqref{inner-term-1} as $k\in K \to \infty$, and using 
the relations $\lim_{k\in K\to\infty}s^k=0$ and the semicontinuity of $\cN_\cX(\cdot)$, we see that the 
first relation of \eqref{opt-cond-c} holds with $z^*=I_J\varpi^*$. By a similar argument as above, one can 
show that $\varpi^*_{I^*}=0$. This together with the definitions of $J^*$ and $\bar J^*$ implies that $z^*$ 
satisfies 
\[
z^*_j=\left\{\ba{ll}
 0 & \ \mbox{if} \ j \in \bar J \cup J^*, \\
 \varpi^*_i & \ \mbox{if} \ j=J(i) \in \bar J^*,
\ea \right.
\] 
where $\bar J$ is the complement of $J$ in $\{1,\ldots,n\}$. In addition, we see from \eqref{lag-multip1} 
that $\lambda^k_i \ge 0$ and $\lambda^k_i g_i(x^k)=0$ for all $i$, which immediately lead to the second 
relation of \eqref{opt-cond-c}. Hence, $(\lambda^*,\mu^*,\varpi^*)$ together with $x^*$ satisfies 
\eqref{opt-cond-c}. Suppose now that $\|x^*_J\|_0=r$. Then, 
$\cJ^* = \{\tilde J^* \subseteq J: |\tilde J^*| = r, x^*_j=0, \forall j \notin \tilde J^*\} = \{J^*\}$.
Therefore, the assumptions of Theorem \ref{suf-cond-thm1} hold. It then follows from Theorem \ref{suf-cond-thm1} 
that $x^*$ is a local minimizer of \eqnok{l0-J1}.
\end{proof}

\subsection{Penalty decomposition method for problem \eqnok{l0-J2}} 
\label{sec:PD-o} 

In this subsection we propose a PD method for solving problem \eqnok{l0-J2} 
and establish some convergence results for it.

We observe that problem \eqnok{l0-J2} can be equivalently reformulated as 
\beq \label{l0-J2-ref}
\min\limits_{x\in \cX,y\in\Re^{|J|}} \{f(x) + \nu \|y\|_0: \ g(x) \le 0, \ h(x)=0, \ \ x_J - y = 0 \}.
\eeq   
The associated quadratic penalty function for \eqnok{l0-J2-ref} is defined as
\beq \label{p-fun}
p_{\vrho}(x,y) := f(x) + \nu \|y\|_0 + \frac{\vrho}{2}(\|[g(x)]^+\|^2+\|h(x)\|^2 + \|x_J - y\|^2) 
\quad \forall x\in \cX,y\in\Re^{|J|}
\eeq
for some penalty parameter $\vrho>0$. 

We are now ready to present the PD method for solving \eqnok{l0-J2-ref} 
(or, equivalently, \eqnok{l0-J2}) in which each penalty subproblem is approximately solved by a BCD method.  

\gap

\noindent
{\bf Penalty decomposition method for \eqnok{l0-J2}:}  \\ [5pt]
Let $\{\epsilon_k\}$ be a positive decreasing sequence. Let $\vrho_0 >0$, $\sigma > 1$ be given, and 
let $q_\vrho$ be defined in \eqref{q-fun}. Choose an arbitrary $y^0_0 \in \Re^{|J|}$ and a constant 
$\Upsilon$ such that $\Upsilon \ge \max\{f(x^{\feas})+\nu \|x^\feas\|_0, \min_{x\in \cX} p_{\vrho_0}(x,y^0_0)\}$. 
Set $k=0$. 
\begin{itemize}
\item[1)] Set $l=0$ and apply the BCD method to find an approximate solution 
$(x^k, y^k) \in \cX \times \Re^{|J|}$ for the penalty subproblem 
\beq \label{inner-prob2}
\min\{p_{\vrho_k}(x,y): \ x \in \cX, \ y \in \Re^{|J|}\} 
\eeq
by performing steps 1a)-1d): 
\bi
\item[1a)] Solve $x^{k}_{l+1} \in \Arg\min\limits_{x \in \cX} p_{\vrho_k}(x,y^k_{l})$.
\item[1b)]
Solve $y^k_{l+1} \in \Arg\min\limits_{y \in \Re^{|J|}} p_{\vrho_k}(x^k_{l+1},y)$.
\item[1c)] Set $(x^k, y^k) := (x^k_{l+1},y^k_{l+1})$. 
If $(x^k,y^k)$ satisfies 
\beq \label{inner-cond-o}
\|\cP_\cX(x^k-\nabla_x q_{\vrho_k}(x^k,y^k))-x^k\|  \le  \epsilon_k, 
\eeq
then go to step 2). 
\item[1d)] Set $l \leftarrow l+1$ and go to step 1a). 
\ei
\item[2)]
Set $\vrho_{k+1} := \sigma\vrho_k$.
\item[3)] 
If $\min\limits_{x \in \cX} p_{\vrho_{k+1}}(x,y^k) > \Upsilon$, set 
$y^{k+1}_0 := x^{\feas}$. Otherwise, set $y^{k+1}_0 := y^k$.
\item[4)]
Set $k \leftarrow k+1$ and go to step 1). 
\end{itemize}
\noindent
{\bf end}

\vgap 

\begin{remark}
The practical termination criteria proposed in Subsection \ref{sec:PD-c} can also be 
applied to this PD method. In addition, one can apply a similar strategy as mentioned in 
Subsection \ref{sec:PD-c} to enhance the performance of the BCD method for solving 
\eqref{inner-prob2}. Finally, in view of Proposition \ref{prop2}, the BCD subproblem 
in step 1b) has a closed-form solution. 
\end{remark}

\vgap

We next establish a convergence result regarding the inner iterations of the above 
PD method. In particular, we will show that an approximate solution 
$(x^k,y^k)$ of problem \eqref{inner-prob2} satisfying \eqref{inner-cond-o}  
can be found by the BCD method described in steps 1a)-1d). For convenience 
of presentation, we omit the index $k$ from \eqnok{inner-prob2} 
and consider the BCD method for solving the following problem:
\beq \label{pd-subprob-o}
\min\{p_\vrho(x,y): \ x \in \cX, \ y \in \Re^{|J|}\}
\eeq
instead. Accordingly, we rename the iterates of the above BCD method. We 
can observe that the resulting BCD method is the same as the one presented 
in Subsection \ref{sec:PD-c} except that $p_\vrho$ and $\Re^{|J|}$ replace 
$q_\vrho$ and $\cY$, respectively. For the sake of brevity, we omit the 
presentation of this BCD method.

\vgap

\begin{lemma} \label{local-soln-o}
Suppose that $(x^*,y^*)\in \Re^n \times \Re^{|J|}$ is a saddle point of problem 
\eqref{pd-subprob-o}, that is,   
\beq \label{saddle-pt-o}
x^* \in \Arg\min\limits_{x \in \cX} p_{\vrho}(x,y^*),  \ \ \ 
y^* \in \Arg\min\limits_{y \in \Re^{|J|}} p_{\vrho}(x^*,y).
\eeq 
Furthermore, assume that $h$'s are affine functions, and $f$ and $g$'s 
are convex functions. Then, $(x^*,y^*)$ is a local minimizer of 
problem \eqref{pd-subprob-o}.
\end{lemma}

\begin{proof}
Let $K =\{i: y^*_i \neq 0\}$, and let $h_x$, $h_y$ 
be any two vectors such that $x^*+h_x \in \cX$, $|(h_y)_i| < \nu/(\rho|x^*_{J(i)}|+1)$ for 
any $i\notin K$ and $|(h_y)_i| < |y^*_i|$ for all $i\in K$. 
We observe from the second relation of \eqref{saddle-pt-o} and Proposition \ref{prop2} that 
$y^*_i = x^*_{J(i)}$ for all $i\in K$. Also, for the above choice of $h_y$, one has
$y^*_i+(h_y)_i \neq 0$ for all $i\in K$. Hence, $\|y^*_i+(h_y)_i\|_0 = \|y^*_i\|_0$ for every
$i\in K$. Using these relations and the definition of $h_y$, we can see that 
\beq \label{aux-ineq}
\rho (y^*-x^*_J)^Th_y + \nu \|y^*+h_y\|_0 - \nu\|y^*\|_0 \ = \ -\rho 
\sum\limits_{i\notin K} x^*_{J(i)}(h_y)_i + \nu \sum\limits_{i\notin K} \|(h_y)_i\|_0  
\ \ge \ 0.
\eeq
In addition, by the assumption that $h$'s are affine functions, and $f$ and $g$'s are convex 
functions, we know that $q_\vrho$ is convex, where $q_\vrho$ is defined in \eqref{q-fun}. 
Also, notice that $p_\vrho(x,y)=q_\vrho(x,y)+\nu\|y\|_0$. It then follows from the first 
relation of \eqref{saddle-pt-o} and the first-order optimality condition that 
$[\nabla_x q_\vrho(x^*,y^*)]^T h_x \ge 0$. Using this relation along with \eqref{aux-ineq} 
and the convexity of $q_\vrho$, we have   
\[
\ba{lcl}
p_\vrho(x^*+h_x,y^*+h_y) &=&  q_\vrho(x^*+h_x,y^*+h_y) + \nu\|y^*+h_y\|_0 \\ [4pt]
& \ge & q_\vrho(x^*,y^*) + [\nabla_x q_\vrho(x^*,y^*)]^T h_x 
+[\nabla_y q_\vrho(x^*,y^*)]^T h_y + \nu\|y^*+h_y\|_0 \\ [4pt]
& \ge & p_\vrho(x^*,y^*) + \vrho(y^*-x^*_J)^T h_y + \nu\|y^*+h_y\|_0 - \nu\|y^*\|_0 \ \ge \ 
p_\vrho(x^*,y^*),
\ea
\]
which together with the above choice of $h_x$ and $h_y$ implies that $(x^*,y^*)$ 
is a local minimizer of \eqref{pd-subprob-o}.
\end{proof}

\vgap


\begin{theorem} \label{BCD-thm-o}
Let $\{(x^l,y^l)\}$ be the sequence generated by the above BCD method, and let $\epsilon >0$ 
be given. Suppose that $(x^*,y^*)$ is an accumulation point of $\{(x^l,y^l)\}$. Then 
the following statements hold:
\bi
\item[(a)] $(x^*,y^*)$ is a saddle point of problem \eqref{pd-subprob-o}.
\item[(b)] There exists some $l>0$ such that 
\[
\|\cP_{\cX}(x^l-\nabla_x q_\vrho(x^l,y^l))-x^l\| < \epsilon,
\]
where the function $q_\vrho$ is defined in \eqref{q-fun}. 
\item[(c)] 
Furthermore, if $h$'s are affine functions, and $f$ and $g$'s are convex functions, then 
$(x^*,y^*)$ is a local minimizer of problem \eqref{pd-subprob-o}.
\ei 
\end{theorem}

\begin{proof} 
We first show that statement (a) holds. Indeed, one can observe that 
\beqa
p_\vrho(x^{l+1},y^l) & \le & p_\vrho(x,y^l) \ \ \ \forall x \in \cX, \label{x-update} \\
p_\vrho(x^{l},y^{l}) & \le & p_\vrho(x^{l},y) \ \ \ \forall y \in \Re^{|J|}. \label{y-update}
\eeqa
It follows that 
\beq \label{squeeze}
p_\vrho(x^{l+1},y^{l+1}) \ \le \ p_\vrho(x^{l+1},y^l) \ \le \ p_\vrho(x^l,y^l) \ \ \ \forall l \ge 1.
\eeq
Hence, the sequence $\{p_\vrho(x^{l},y^{l})\}$ is non-increasing. Since $(x^*,y^*)$ is an accumulation 
point of $\{(x^l,y^l)\}$, there exists a subsequence $L$ such that $\lim_{l\in L \to \infty} (x^l,y^l) = (x^*,y^*)$, 
and moreover, $x^* \in \cX$ due to the closedness of $\cX$.   
We then observe from \eqref{p-fun} that $\{p_\vrho(x^{l},y^{l})\}_{l\in L}$ is bounded, which together with the 
monotonicity of $\{p_\vrho(x^{l},y^{l})\}$ implies that $\{p_\vrho(x^{l},y^{l})\}$ is bounded below and hence 
$\lim_{l\to\infty}p_\vrho(x^{l},y^{l})$ exists. This observation and \eqref{squeeze} yield 
\beq \label{eq-limits}
\lim_{l\to\infty} p_\vrho(x^{l},y^{l}) = \lim_{l\to\infty} p_\vrho(x^{l+1},y^l).
\eeq
For notational convenience, let 
\[
F(x) := f(x) + \frac{\vrho}{2}(\|[g(x)]^+\|^2+\|h(x)\|^2). 
\]
It then follows from \eqref{p-fun} that 
\beq \label{p-fun-F}
p_{\vrho}(x,y) = F(x) + \nu \|y\|_0 + \frac{\vrho}{2}\|x_J - y\|^2,  \quad \forall x\in \cX,y\in\Re^{|J|}.
\eeq
Since $\lim_{l\in L} y^l = y^*$, one has $\|y^l\|_0 \ge \|y^*\|_0$ for sufficiently 
large $l\in L$. Using this relation, \eqref{y-update} and \eqref{p-fun-F}, we obtain that, 
when $l \in L$ is sufficiently large,
\[
p_\vrho(x^{l},y) \ \ge \ p_\vrho(x^{l},y^{l}) \ = \ F(x^l) + \nu \|y^l\|_0 + 
\frac{\vrho}{2} \|x^l_J - y^l\|^2 \ \ge \ F(x^l) + \nu \|y^*\|_0 + 
\frac{\vrho}{2} \|x^l_J - y^l\|^2. 
\]
Upon taking limits on both sides of the above inequality as $l \in L\to \infty$ 
and using the continuity of $F$, one has 
\beq \label{sad-pt1}
p_\vrho(x^*,y) \ \ge \ F(x^*) + \nu \|y^*\|_0 + \frac{\vrho}{2} \|x^*_J - y^*\|^2 
\ = \  p_\vrho(x^*,y^*), \ \ \ \forall y\in \Re^{|J|}.
\eeq
In addition, it follows from \eqref{x-update} and \eqref{p-fun-F} that 
\beq \label{f-ineq}
\ba{lcl}
F(x)+\frac12\|x_{J}-y^l\|^2 &=& p_\vrho(x,y^l) - \nu \|y^l\|_0 
\ \ge \  p_\vrho(x^{l+1},y^l) - \nu \|y^l\|_0 \\ [4pt]
&=& \ F(x^{l+1})+\frac12\|x^{l+1}_J-y^l\|^2, \ \ \ \forall x \in \cX.
\ea
\eeq
Since $\{\|y^l\|_0\}_{l\in L}$ is bounded, there exists a subsequence $\bar L \subseteq L$ 
such that $\lim_{l \in \bar L \to \infty} \|y^l\|_0$ exists. Then we have 
\beq \nn
\ba{l}
 \lim\limits_{l \in \bar L \to \infty} F(x^{l+1})+\frac12\|x^{l+1}_J-y^l\|^2 \ = \  
\lim\limits_{l \in \bar L \to \infty} p_\vrho(x^{l+1},y^l) - \nu \|y^l\|_0 \\ [12pt]
=  \lim\limits_{l \in \bar L \to \infty} p_\vrho(x^{l+1},y^l) - \nu\lim\limits_{l \in \bar L \to \infty} \|y^l\|_0 
\ = \ \lim\limits_{l \in \bar L\to \infty} p_\vrho(x^l,y^l) - \nu\lim\limits_{l \in \bar L \to \infty} \|y^l\|_0 \\ [12pt]
 =   \lim\limits_{l \in \bar L \to \infty} p_\vrho(x^{l},y^l) - \nu \|y^l\|_0
= \lim\limits_{l \in \bar L \to \infty} F(x^{l})+\frac12\|x^{l}_J-y^l\|^2  \ = \ F(x^*)+\frac12\|x^*_J-y^*\|^2,
\ea
\eeq
where the third equality is due to \eqref{eq-limits}.
Using this relation and taking limits on both sides of \eqref{f-ineq} as $l \in \bar L \to \infty$, 
we further have
\[
F(x)+\frac12\|x_{J}-y^*\|^2  \ \ge \ F(x^*)+\frac12\|x^*_J-y^*\|^2, \ \ \ \forall x \in \cX, 
\] 
which together with \eqref{p-fun} yields 
\[
p_\vrho(x,y^*) \ \ge \ p_\vrho(x^*,y^*), \ \ \ \forall x \in \cX. 
\]
This relation along with \eqref{sad-pt1} implies that $(x^*,y^*)$ is a saddle point of 
\eqref{pd-subprob-o} and hence statement (a) holds.  Statement (b) can be similarly proved 
as that of Theorem \ref{BCD-thm-c}. In addition, statement (c) holds due to statement (a) 
and Lemma \ref{local-soln-o}.  
\end{proof}

\vgap

\begin{remark}
A similar result as in statement (c) is recently established in \cite{ZhDoLu11} for the 
BCD method when applied to solve the {\it unconstrained} problem: 
\beq \label{pvrho}
\min\limits_{x,y} \frac12\|Ax-b\|^2 + 
\frac{\vrho}{2}\|Wx-y\|^2 + \sum_i \nu_i \|y_i\|_0
\eeq
under the assumption that $A^TA\succ 0$, $W^TW=I$, $\vrho >0$, and $\nu_i \ge 0$ for all $i$. 
The proof of \cite{ZhDoLu11} strongly relies on this assumption and the fact that the BCD 
subproblems have closed-form solutions. We believe that it cannot be extended to problem 
\eqref{pd-subprob-o}. In addition, it is not hard to observe that problem \eqref{pvrho} 
can be equivalently reformulated into a problem in the form of \eqref{pd-subprob-o} and 
thus the convergence of the BCD method for \eqref{pvrho} directly follows from Theorem 
\ref{BCD-thm-o} above. 
\end{remark}

\vgap

We next establish the convergence of the outer iterations of the PD method for solving 
problem \eqref{l0-J2}. In particular, we show that under some suitable assumption, any 
accumulation point of the sequence generated by the PD method satisfies the first-order 
optimality conditions of \eqref{l0-J2}.  Moreover, when the $l_0$ part is the only nonconvex 
part, we show that the accumulation point is a local minimizer of \eqref{l0-J2}. 

\begin{theorem} \label{main-thm-o}
Assume that $\epsilon_k \to 0$. Let $\{(x^k,y^k)\}$ be the sequence generated by the 
above PD method. Suppose that the level set 
$\cX_{\Upsilon} := \{x \in \cX: f(x) \le \Upsilon \}$ is compact. Then, the following 
statements hold:
\bi
\item[(a)] The sequence $\{(x^k,y^k)\}$ is bounded; 
\item[(b)] Suppose $(x^*,y^*)$ is an accumulation point of $\{(x^k,y^k)\}$. Then, 
$x^*=y^*$ and $x^*$ is a feasible point of problem \eqnok{l0-J2}. 
\item[(c)] Let $(x^*,y^*)$ be defined above. Suppose that $\{(x^k,y^k)\}_{k \in K} \to (x^*,y^*)$ 
for some subsequence $K$. Let $J^* = \{j\in J: x^*_j \neq 0\}$, $\bar J^*=J\setminus J^*$. 
Assume that the Robinson condition \eqref{rob-cond-c} holds at $x^*$ for such $\bar J^*$. 
Then, $\{(\lambda^{k},\mu^{k},\varpi^k)\}_{k\in K}$
is bounded, where 
\[
\lambda^{k} = \vrho_k [g(x^k)]^+, \ \ \ \
\mu^{k} = \vrho_k h(x^k), \ \ \ \ \varpi^k = \vrho_k(x^k_{J}-y^k).
\]
Moreover, each accumulation point $(\lambda^*,\mu^*,\varpi^*)$ of 
$\{(\lambda^k,\mu^k,\varpi^k)\}_{k\in K}$ together with $x^*$ 
satisfies the first-order optimality condition \eqref{opt-cond-c}
with $z^*_j=\varpi^*_i$ for all $j=J(i) \in \bar J^*$. Further, if 
$h$'s are affine functions, and $f$ and $g$'s are convex 
functions, then $x^*$ is a local minimizer of problem \eqnok{l0-J2}. 
\ei
\end{theorem} 

\begin{proof}
Statement (a) and (b) can be similarly proved as those of Theorem \ref{main-thm-c}. We 
now show that statement (c) holds. Let $I^* = \{i: J(i) \in J^*\}$. From Proposition 
\ref{prop2} and the definitions of $y^k$ and $J^*$, we can observe that 
$y^k_{I^*}= x^k_{J^*}$ when $k\in K$ is sufficiently large. Hence, $\varpi^k_{I^*}=0$ 
for sufficiently large $k\in K$. The rest of the proof for the first two conclusions of this 
statement is similar to that of statement (c) of Theorem \ref{main-thm-c}. The 
last conclusion of this statement holds due to its second conclusion and Theorem \ref{suf-cond-thm2}.  
\end{proof}

\vgap

\section{Numerical results}
\label{results}

In this section, we conduct numerical experiments to test the performance of our PD methods proposed 
in Section \ref{method} by applying them to sparse logistic regression, sparse inverse covariance 
selection, and compressed sensing problems. The codes of all the methods implemented in this section 
are written in Matlab, which are available online at www.math.sfu.ca/$\sim$zhaosong. All experiments 
are performed in Matlab 7.11.0 (2010b) on a workstation with an Intel Xeon E5410 CPU (2.33 GHz) and 
8GB RAM running Red Hat Enterprise Linux (kernel 2.6.18).
  
\subsection{Sparse logistic regression problem}
\label{logistic}

In this subsection, we apply the PD method studied in Subsection \ref{sec:PD-c} to solve sparse 
logistic regression problem, which has numerous applications in machine learning, computer vision, 
data mining, bioinformatics, and neural signal processing (see, for example, \cite{Bi07,TsMcTsAn07,
LiCh07,PaSpGeSa05,GePaSa05,PhSa06}). 

Given $n$ samples $\{z^1, \ldots, z^n\}$ with $p$ features, and $n$ binary outcomes 
$b_1, \ldots, b_n$, let $a^i = b_iz^i$ for $i=1,\ldots,n$. The {\it average logistic loss} 
function is defined as 
\[
 l_\avg(v,w) := \sum_{i=1}^n\theta(w^Ta^i+vb_i)/n
\]
for some model variables $v \in \Re$ and $w \in \Re^p$, where $\theta$ is the 
{\it logistic loss} function 
$$\theta(t) := \log(1+\exp(-t)).$$ Then the {\it sparse logistic regression} problem 
can be formulated as 
\beq
\min\limits_{v,w}  \left\{l_\avg(v,w): \|w\|_0 \leq r\right\}, \label{Prob1} 
\eeq
where $r \in [1,p]$ is some integer for controlling the sparsity of the solution. 
In the literature, one common approach for finding an approximate solution to 
\eqref{Prob1} is by solving the following $l_1$ regularization problem: 
\beq \label{l1-reg}
\min\limits_{v,w}  l_\avg(v,w) + \lambda \|w\|_1
\eeq
for some regularization parameter $\lambda \ge 0$ (see, for example, \cite{KoKiBo07,
EfHaJoTi04,PaHa06,LeLeAbNg06,LiJiYe09,ShYiOsSa10}). Our aim below is to apply the 
PD method studied in Subsection \ref{sec:PD-c} to solve \eqnok{Prob1} directly.  

Letting $x = (v,w)$, $J=\{2,\ldots,p+1\}$ and $f(x) = l_\avg(x_1,x_J)$, we can see that 
problem \eqnok{Prob1} is in the form of \eqnok{l0-J1}. Therefore, the PD method proposed 
in Subsection \ref{sec:PD-c} can be suitably applied to solve \eqnok{Prob1}. Also, 
we observe that the main computation effort of the PD method when applied to \eqnok{Prob1} 
lies in solving the subproblem arising in step 1a), which is in the form of
\beq
\min_{x} \left\{l_{\avg}(x_1,x_J) + \frac{\vrho}{2} \|x-c\|^2: \ x \in \Re^{p+1}\right\} 
 \label{Prob1-1}\\
\eeq
for some $\vrho>0$ and $c \in \Re^{p+1}$. To efficiently solve \eqnok{Prob1-1}, we apply 
the nonmonotone projected gradient method proposed in \cite[Algorithm 2.2]{BiMaRa00}; in 
particular, we set its parameter $M=2$ and terminate the method when 
${\|\nabla F(x)\|}/{\max\{|F(x)|,1\}} \le 10^{-4}$, where $F(x)$ denotes the 
objective function of \eqnok{Prob1-1}.



We now address the initialization and the termination criteria for our PD method when applied to \eqnok{Prob1}. 
In particular, we randomly generate $z\in\Re^{p+1}$ such that $\|z_J\|_0 \le r$ and set the initial point 
$y^0_0 = z$.  We choose the initial penalty parameter $\vrho_0$ to be $0.1$, and set the parameter 
$\sigma = \sqrt{10}$. 
In addition, we use \eqnok{inner-term} and \eqnok{outer-term} as the inner and outer termination criteria 
for the PD method and set their accuracy parameters $\eps_I$ and $\eps_O$ to be $5\times 10^{-4}$ and
$10^{-3}$, respectively.

We next conduct numerical experiments to test the performance of our PD method for solving \eqnok{Prob1} 
on some real and random data. We also compare the quality of the approximate solutions of \eqnok{Prob1} 
obtained by our method with that of \eqnok{l1-reg} found by a first-order solver SLEP \cite{LiJiYe09}. 
For the latter method, we set opts.mFlag=1, opts.lFlag=1 and opts.tFlag=2. And the rest of its 
parameters are set by default. 

In the first experiment, we compare the solution quality of our PD method with SLEP on three small- 
or medium-sized benchmark data sets which are from the UCI machine learning bench market repository \cite{NeHeBlMe98} 
and other sources \cite{GoVa96}. The first data set is the colon tumor gene expression data \cite{GoVa96} with more 
features than samples; the second one is the ionosphere data \cite{NeHeBlMe98} with less features than samples; 
and the third one is the Internet advertisements data \cite{NeHeBlMe98} with roughly same magnitude of features 
as samples. We discard the samples with missing data and standardize each data set so that the sample mean is 
zero and the sample variance is one. For each data set, we first apply SLEP to solve problem \eqnok{l1-reg} with 
four different values of $\lambda$, which are the same ones as used in \cite{KoKiBo07}, namely, 
$0.5\lambda_{\max}$, $0.1\lambda_{\max}$, $0.05\lambda_{\max}$, and $0.01\lambda_{\max}$, where 
$\lambda_{\max}$ is the upper bound on the useful range of $\lambda$ that is defined in \cite{KoKiBo07}. 
For each such $\lambda$, let $w^*_\lambda$ be the approximate optimal $w$ obtained by SLEP. We then 
apply our PD method to solve problem \eqnok{Prob1} with $r=\|w^*_\lambda\|_0$ so that the resulting 
approximate optimal $w$ is at least as sparse as $w^*_\lambda$. 

To compare the solution quality of the above two methods, we introduce a criterion, that is, {\it error rate}. Given 
any model variables $(v,w)$ and a sample vector $z\in \Re^p$, the outcome predicted by 
$(v,w)$ for $z$ is given by
\[
\phi(z) = \sgn(w^T z +v),
\] 
where 
\[
\sgn(t) = \left\{\ba{ll}
+1 & \mbox{if} \ t > 0, \\
-1 & \mbox{otherwise}.  
\ea 
\right. 
\]
Recall that $z^i$ and $b_i$ are the given samples and outcomes for $i=1,\ldots, n$. The {\it error rate} of $(v,w)$ 
for predicting the outcomes $b_1, \ldots, b_n$ is defined as 
\[
\mbox{Error} :=\left\{ \sum^{n}_{i=1} \|\phi(z^i)-b_i\|_0/n \right\} \times 100\%. 
\] 
   
The computational results are presented in Table \ref{rdata-log}. In detail, the name and dimensions of each 
data set are given in the first three columns. The fourth column gives the ratio between $\lambda$ and its 
upper bound $\lambda_{\max}$. The fifth column lists the value of $r$, that is, the cardinality of 
$w^*_\lambda$ which is defined above. In addition, the average logistic loss, the error rate and the CPU 
time (in seconds) for both SLEP and PD are reported in columns six to eleven. We can observe 
that, although SLEP is faster than the PD method in most cases, the PD method substantially 
outperforms SLEP in terms of the solution quality since it generally achieves lower average 
logistic loss and error rate while the sparsity of both solutions is the same. 

\begin{table}[t!]
\caption{\small  Computational results on three real data sets} 
\centering
\label{rdata-log}
\begin{footnotesize}
\begin{tabular}{|c|c|c|c|c|c c c|c c c|}
\hline 
\multicolumn{1}{|c|}{Data} & \multicolumn{1}{c|}{Features } & 
\multicolumn{1}{c|}{Samples } &
\multicolumn{1}{c|}{} &
\multicolumn{1}{c|}{} & \multicolumn{3}{c|}{SLEP} &
\multicolumn{3}{c|}{PD}
\\
\multicolumn{1}{|c|}{} & \multicolumn{1}{c|}{$p$} & 
\multicolumn{1}{c|}{$n$} & \multicolumn{1}{c|}{$\lambda / \lambda_{\max}$} &
\multicolumn{1}{c|}{$r$} & \multicolumn{1}{c}{$l_{\avg}$} & \multicolumn{1}{c}{Error $(\%)$} & \multicolumn{1}{c|}{Time} &
\multicolumn{1}{c}{$l_{\avg}$} & \multicolumn{1}{c}{Error $(\%)$} & \multicolumn{1}{c|}{Time} 
\\
\hline
Colon                & $2000$ & $62$   &  $0.5$   & $7$    & $0.4398$ & $17.74$ & $0.2$ & $0.4126$ & $12.9$ & $9.1$ \\
                         &              &            &  $0.1$   & $22$  & $0.1326$ & $1.61$   & $0.5$ & $0.0150$ & $0$        & $6.0$\\
                         &              &            &  $0.05$ & $25$  & $0.0664$ & $0$        & $0.6$ & $0.0108$ & $0$        & $5.0$\\
                         &              &            &  $0.01$ & $28$  & $0.0134$ & $0$        & $1.3$ & $0.0057$ & $0$        & $5.4$\\
\hline
Ionosphere       & $34$     & $351$  &  $0.5$  & $3$   & $0.4804$ & $17.38$ & $0.1$  & $0.3466$ & $13.39$ & $0.7$\\
                        &              &             &  $0.1$  & $11$  & $0.3062$ & $11.40$ & $0.1$ & $0.2490$ & $9.12$   &  $1.0$\\
                        &              &            &  $0.05$ & $14$  & $0.2505$ & $9.12$  & $0.1$  & $0.2002$ & $8.26$   &  $1.1$\\
                        &              &            &  $0.01$ & $24$  & $0.1846$ & $6.55$  & $0.4$  & $0.1710$ & $5.98$   &  $1.7$\\
\hline
Advertisements  & $1430$ & $2359$ &  $0.5$  & $3$     & $0.2915$ & $12.04$ & $2.3$    & $0.2578$ & $7.21$ & $31.9$\\
                          &              &              &  $0.1$  & $36$    & $0.1399$ & $4.11$  & $14.2$   & $0.1110$ & $4.11$ & $56.0$\\
                          &              &              &  $0.05$ & $67$   & $0.1042$ & $2.92$  & $21.6$   & $0.0681$ & $2.92$ & $74.1$\\
                          &              &              &  $0.01$ & $197$ & $0.0475$ & $1.10$  & $153.0$ & $0.0249$ & $1.10$ & $77.4$\\
\hline
\end{tabular}
\end{footnotesize}
\\
\end{table}

In the second experiment, we test our PD method on the random data sets of three different sizes. For each size, 
we randomly generate the data set consisting of $100$ instances. In particular, the first data set has more 
features than samples; the second data set has more samples than features; and the last data set  
has equal number of features as samples. The samples $\{z^1, \ldots, z^n\}$ and the corresponding 
outcomes $b_1, \ldots, b_n$ are generated in the same manner as described in 
\cite{KoKiBo07}. In detail, for each instance we choose equal number of positive and negative 
samples, that is, $m_+  =m_- = m/2$, where $m_+$ (resp., $m_-$) is the number of samples with outcome 
$+1$ (resp., $-1$). The features of positive (resp., negative) samples are independent and identically distributed, 
drawn from a normal distribution $N(\mu,1)$, where $\mu$ is in turn drawn from a uniform distribution on $[0,1]$ 
(resp., $[-1,0]$). For each such instance, similar to the previous experiment, we first apply SLEP to solve problem 
\eqnok{l1-reg} with five different values of $\lambda$, which are $0.9\lambda_{\max}$, $0.7\lambda_{\max}$, 
$0.5\lambda_{\max}$, $0.3\lambda_{\max}$ and $0.1\lambda_{\max}$. For each such $\lambda$, let $w^*_\lambda$ be 
the approximate optimal $w$ obtained by SLEP. We then apply our PD method to solve problem \eqnok{Prob1} with 
$r=\|w^*_\lambda\|_0$ so that the resulting approximate optimal $w$ is at least as sparse as $w^*_\lambda$. 
The average results of each data set over $100$ instances are reported in Table \ref{rdata-log1}. 
We also observe that the PD method is slower than SLEP, but it has better solution quality than SLEP in terms of 
average logistic loss and error rate.

In summary, the above experiments demonstrate that the quality of the approximate 
solution of \eqnok{Prob1} obtained by our PD method is generally better than 
that of \eqnok{l1-reg} found by SLEP when the same sparsity is considered. This observation 
is actually not surprising as \eqnok{l1-reg} is a relaxation of \eqnok{Prob1}.   

\begin{table}[t!]
\caption{\small  Computational results on random data sets} 
\centering
\label{rdata-log1}
\begin{footnotesize}
\begin{tabular}{|c|c|c|c c c|c c c|}
\hline 
\multicolumn{1}{|c|}{Size} &
\multicolumn{1}{c|}{} &
\multicolumn{1}{c|}{} & \multicolumn{3}{c|}{SLEP} &
\multicolumn{3}{c|}{PD}
\\
\multicolumn{1}{|c|}{$n \times p$}  & \multicolumn{1}{c|}{$\lambda / \lambda_{\max}$} &
\multicolumn{1}{c|}{$r$} & \multicolumn{1}{c}{$l_{\avg}$} & \multicolumn{1}{c}{Error $(\%)$} & \multicolumn{1}{c|}{Time} &
\multicolumn{1}{c}{$l_{\avg}$} & \multicolumn{1}{c}{Error $(\%)$} & \multicolumn{1}{c|}{Time} 
\\
\hline
$1000 \times 2000$   &  $0.9$   & $17.0$  & $0.6411$ & $9.76$  & $0.4$ & $0.2145$ & $8.49$  & $9.9$ \\
                                   &  $0.7$   & $52.9$  & $0.5090$ & $3.96$   & $1.0$ & $0.0588$ & $2.66$  & $20.0$\\
                                   &  $0.5$ & $96.6$    & $0.3838$ & $2.23$   & $1.7$ & $0.0060$ & $0.02$  & $34.9$\\
                                   &  $0.3$ & $138.7$  & $0.2611$ & $1.22$   & $2.1$ & $0.0022$ & $0$       & $25.5$\\
                                   &  $0.1$ & $192.0$  & $0.1228$ & $0.31$   & $2.0$ & $0.0013$ & $0$       & $16.0$\\
\hline
$2000 \times 1000$   &  $0.9$   & $11.0$    & $0.6441$ & $11.46$ & $0.4$ & $0.2763$ & $10.67$ & $15.2$ \\
                               &  $0.7$   & $42.8$    & $0.5083$ & $3.63$        & $1.1$ & $0.0376$ & $1.49$   & $38.9$\\
                               &  $0.5$   & $78.0$    & $0.3776$ & $1.65$        & $2.0$ & $0.0032$ & $0$        & $34.4$\\
                               &  $0.3$   & $115.5$  & $0.2490$ & $0.6$          & $2.6$ & $0.0015$ & $0$        & $25.3$\\
                               &  $0.1$   & $160.8$  & $0.1056$ & $0.03$        & $3.1$ & $0.0010$ & $0$        & $15.8$\\
\hline
$1000 \times 1000$   &  $0.9$   & $11.7$    & $0.6417$ & $11.00$ & $0.1$ & $0.2444$ & $9.67$ & $2.3$ \\
                                   &  $0.7$   & $37.2$    & $0.5086$ & $3.95$   & $0.2$ & $0.0572$ & $2.46$   & $5.8$\\
                                   &  $0.5$   & $67.6$    & $0.3805$ & $2.15$   & $0.3$ & $0.0060$ & $0.01$   & $6.2$\\
                                   &  $0.3$   & $100.1$  & $0.2544$ & $0.81$   & $0.4$ & $0.0016$ & $0$        & $4.6$\\
                                   &  $0.1$   & $137.9$  & $0.1124$ & $0.12$   & $0.5$ & $0.0011$ & $0$        & $3.3$\\
\hline
\end{tabular}
\end{footnotesize}
\\
\end{table}

\subsection{Sparse inverse covariance selection problem} 
\label{covariance}

In this subsection, we apply the PD method proposed in Subsection \ref{sec:PD-c} to solve the sparse inverse covariance 
selection problem, which has numerous real-world applications such as speech recognition and gene network analysis 
(see, for example, \cite{Bi00,DoHaJoNeYa04}).

Given a sample covariance matrix $\bsigma\in \cS^{p}_{++}$ and a set $\Omega$ consisting of pairs of known conditionally 
independent nodes, the sparse inverse covariance selection problem can be formulated as  
\beq \label{Prob2}
\ba{rl}
\max\limits_{X \succeq 0} & \log \det  X - \left\langle {\bsigma,X} \right\rangle \\
\mbox{s.t.} 
& \sum\limits_{(i,j) \in \bomega} \|X_{ij}\|_0 \leq r, \\ [4pt]
& X_{ij}=0 \quad \forall (i,j) \in \Omega,
\ea
\eeq
where $\bomega =  \left\{(i,j):(i,j) \notin \Omega, \ i \neq j \right\}$, and $r \in [1, |\bomega|]$ is some integer for 
controlling the sparsity of the solution. In the literature, one common approach for finding an approximate solution to 
\eqref{Prob2} is by solving the following $l_1$ regularization problem:
\beq
\label{l1-cov}
\ba{rl}
\max\limits_{X \succeq 0} & \log \det  X - \left\langle {\bsigma,X} \right\rangle -  \sum\limits_{(i,j) \in \bomega} 
\rho_{ij}|X_{ij}|\\
\mbox{s.t.} & X_{ij}=0 \quad \forall (i,j) \in \Omega,
\ea
\eeq 
where $\{\rho_{ij}\}_{(i,j)\in \bomega}$ is a set of regularization parameters (see, for example, 
\cite{DaBaGh08,DaVaRo08,BaGhDa08,Lu09,Lu10,FrHaTi08,WaSuTo09,LiTo10}). Our goal below is to apply the 
PD method studied in Subsection \ref{sec:PD-c} to solve \eqnok{Prob2} directly.

Letting $\cX =  \left\{X\in \cS^p_+: X_{ij} = 0, \ (i,j) \in \Omega\right\}$ and $J = \bomega$, we clearly see 
that problem \eqnok{Prob2} is in the form of \eqnok{l0-J1} and thus it can be suitably solved by the PD method 
proposed in Subsection \ref{sec:PD-c} with 
\[
\cY = \left\{Y\in \cS^p: \sum_{(i,j) \in \bomega} \|Y_{ij}\|_0 \leq r \right\}.
\]
Notice that the main computation effort of the PD method when applied to \eqnok{Prob2} lies in solving the 
subproblem arising in step 1a), which is in the form of 
\beq
\min_{X \succeq 0}  \left\{- \log\det X + \frac{\vrho}{2}\|X-C\|^2_F: X_{ij}=0 \ \forall (i,j) \in \Omega \right\}
\label{Prob2-1} 
\eeq
for some $\vrho >0$ and $C \in \cS^p$. Given that problem \eqnok{Prob2-1} generally does not have a closed-form 
solution, we now slightly modify the above sets $\cX$ and $\cY$ by replacing them by  
\[
\cX = \cS^p_+, \quad \quad \cY = \left\{Y\in \cS^p: \sum_{(i,j) \in \bomega} \|Y_{ij}\|_0 \leq r, \ 
Y_{ij} = 0, \ (i,j) \in \Omega\right\},  
\]
respectively, and then apply the PD method presented in Subsection \ref{sec:PD-c} to solve 
\eqnok{Prob2}. For this PD method, the subproblem arising in step 1a) is now in 
the form of  
\beq
\min_X  \left\{- \log\det X + \frac{\vrho}{2}\|X-C\|^2_F: X \succeq 0 \right\} 
\label{Prob2-2} \\
\eeq
for some $\vrho >0$ and $C \in \cS^p$. It can be shown that problem \eqnok{Prob2-2} has a closed-form 
solution, which is given by $V\D(x^*)V^T$, where $x^*_i = (\lambda_i + \sqrt{\lambda_i^2+4/\vrho})/2$ for 
all $i$ and $V\D(\lambda) V^T$ is the eigenvalue decomposition of $C$ for some $\lambda \in \Re^p$ (see, for example, 
Proposition 2.7 of \cite{LuZh10-1}). Also, it follows from Proposition \ref{prop1} that the subproblem 
arising in step 1b) for the above $\cY$ has a closed-form solution. 
    
We now address the initialization and the termination criteria for the above PD method. 
In particular, we set the initial point $Y_0^0 = (\tD(\bsigma))^{-1}$, the initial penalty parameter $\vrho_0=1$, and the parameter 
$\sigma = \sqrt{10}$. 
In addition, we use \eqnok{outer-term} and 
\[
\frac{|q_{\vrho_k}(x^k_{l+1},y^k_{l+1})- q_{\vrho_k}(x^k_{l},y^k_{l})|}{\max\{|q_{\vrho_k}(x^k_{l},y^k_{l})|,1\}}
\le \eps_I
\] 
as the outer and inner termination criteria for the PD method, and set the associated accuracy parameters 
$\eps_O=10^{-4}$ and $\eps_I=10^{-4}, \ 10^{-3}$ for the random and real data below, respectively.   

We next conduct numerical experiments to test the performance of our PD method for solving \eqnok{Prob2} 
on some random and real data. We also compare the quality of the approximate solutions of 
\eqnok{Prob2} obtained by our method with that of \eqnok{l1-cov} found by the proximal point algorithm (PPA) 
\cite{WaSuTo09}. Both methods call the LAPACK routine {\bf dsyevd.f} \cite{LAP} for computing the full 
eigenvalue decomposition of a symmetric matrix, which is usually faster than the Matlab's {\bf eig} routine 
when $p$ is larger than $500$. For PPA, we set ${\rm Tol} =10^{-6}$ and use the default values for all other 
parameters.

In the first experiment, we compare the solution quality of our PD method with PPA on a set of random instances 
which are generated in a similar manner as described in \cite{DaBaGh08,Lu09,Lu10,WaSuTo09,LiTo10}. In 
particular, we first generate a true covariance matrix $\bsigmat \in \cS^p_{++}$ such that its inverse 
$(\bsigmat)^{-1}$ is with the prescribed density $\delta$, and set 
$$\Omega =  \left\{(i,j): (\bsigmat)^{-1}_{ij} = 0, \ |i-j| \geq \lfloor p/2 \rfloor\right\}.$$ 
We then generate a matrix $B \in \cS^p$ by letting
\[
 B= \bsigmat + \tau V,
\]
where $V \in \cS^p$ contains pseudo-random values drawn from a uniform distribution on the interval $[-1,1]$, 
and $\tau$ is a small positive number. Finally, we obtain the following sample covariance matrix:
\[
 \bsigma = B - \min\{\lambda_{\min}(B) - \vtheta,0\}I,
\]
where $\vtheta$ is a small positive number. Specifically, we choose $\tau=0.15$, $\vtheta = 1.0e-4$, 
$\delta = 10\%$, $50\%$ and $100\%$, respectively. It is clear that for $\delta=100\%$ and 
the set $\Omega$ is an empty set. In addition, for all $(i,j) \in \bomega$, we set $\rho_{ij} = \rho_{\bomega}$ 
for some $\rho_{\bomega}>0$. For each instance, we first apply PPA to solve \eqnok{l1-cov} for four values of 
$\rho_{\bomega}$, which are $0.01$, $0.1$, $1$, and $10$. For each $\rho_{\bomega}$, let 
$\tX^*$ be the solution obtained by PPA. We then apply our PD method to solve problem \eqnok{Prob2} with 
$r=\sum_{(i,j) \in \bomega} \|\tX^*_{ij}\|_0$ so that the resulting solution is at least as sparse as $\tX^*$. 

As mentioned in \cite{LiTo10}, to evaluate how well the true inverse covariance matrix $(\bsigmat)^{-1}$ 
is recovered by a matrix $X \in \cS^p_{++}$, one can compute the {\it normalized entropy loss} which is 
defined as follows:
\[
\mbox{Loss} := \frac{1}{p}(\left\langle {\bsigmat,X} \right\rangle - \log \det (\bsigmat X) -p).
\]
The results of PPA and the PD method on these instances are presented in Tables \ref{rdata-cov1}-\ref{rdata-cov4}, 
respectively. In each table, the order $p$ of $\bsigma$ is given in column one. The size 
of $\Omega$ is given in column two. The values of $\rho_\bomega$ and $r$ are given in columns three and four. 
The log-likelihood (i.e., the objective value of \eqnok{Prob2}), the normalized entropy loss and the CPU time 
(in seconds) of PPA and the PD method are given in the last six columns, respectively. We observe 
that our PD method is substantially faster than PPA for these instances. Moreover, it outperforms PPA in terms of 
solution quality since it achieves larger log-likelihood and smaller normalized entropy loss.

\begin{table}[t!]
\caption{\small  Computational results for $\delta = 10\%$} 
\centering
\label{rdata-cov1}
\begin{footnotesize}
\begin{tabular}{|c c|c|c|c c c|c c c |}
\hline 
\multicolumn{2}{|c|}{Problem }  & \multicolumn{1}{c|}{} &
\multicolumn{1}{c|}{} & \multicolumn{3}{c|}{PPA} &
\multicolumn{3}{c|}{PD}
\\
\multicolumn{1}{|c}{$p$} & \multicolumn{1}{c|}{$|\Omega|$} &   
\multicolumn{1}{c|}{$\rho_{\bomega}$} & \multicolumn{1}{c|}{$r$} & 
\multicolumn{1}{c}{Likelihood}  & \multicolumn{1}{c}{Loss} & \multicolumn{1}{c|}{Time} &
\multicolumn{1}{c}{Likelihood}  & \multicolumn{1}{c}{Loss} & \multicolumn{1}{c|}{Time}
\\
\hline
500 & $56724$ & $0.01$ & $183876$ & $-950.88$ & $2.4594$ & $34.1$& $-936.45$ & $2.3920$ & $2.5$ \\ 

& & $0.10$ & $45018$ & $-999.89$ & $2.5749$ & $44.8$& $-978.61$ & $2.4498$ & $5.3$ \\ 

& &  $1.00$ & $5540$ & $-1046.44$ & $2.9190$ & $66.2$ & $-1032.79$ & $2.6380$ & $24.8$ \\

& & $10.0$ & $2608$ & $-1471.67$ & $4.2442$ & $75.1$& $-1129.50$ & $2.8845$ & $55.5$ \\

\hline 

1000 & $226702$ & $0.01$ & $745470$ & $-2247.14$ & $3.1240$ & $150.2$& $-2220.47$ & $3.0486$ & $13.1$ \\ 

& & $0.10$ & $186602$ & $-2344.03$ & $3.2291$ & $158.7$& $-2301.12$ & $3.1224$ & $19.8$ \\ 

& & $1.00$ & $29110$ & $-2405.88$ & $3.5034$ & $349.8$& $-2371.68$ & $3.2743$ & $59.1$ \\ 

& & $10.0$ & $9604$ & $-3094.57$ & $4.6834$ & $395.9$& $-2515.80$ & $3.4243$ & $129.5$ \\ 

\hline 

1500 & $509978$ & $0.01$ & $1686128$ & $-3647.71$ & $3.4894$ & $373.7$ & $-3607.23$ & $3.4083$ & $35.7$ \\ 

& & $0.10$ & $438146$ & $-3799.02$ & $3.5933$ & $303.6$& $-3731.17$ & $3.5059$ & $44.9$ \\ 

& & $1.00$ & $61222$ & $-3873.93$ & $3.8319$ & $907.4$& $-3832.88$ & $3.6226$ & $155.3$ \\ 

& & $10.0$ & $17360$ & $-4780.33$ & $4.9264$ & $698.8$& $-3924.94$ & $3.7146$ & $328.0$ \\

\hline 

2000 & $905240$ & $0.01$ & $3012206$ & $-5177.80$ & $3.7803$ & $780.0$& $-5126.09$ & $3.7046$ & $65.5$ \\ 

& & $0.10$ & $822714$ & $-5375.21$ & $3.8797$ & $657.5$& $-5282.37$ & $3.7901$ & $94.3$ \\ 

& & $1.00$ & $126604$ & $-5457.90$ & $4.0919$ & $907.4$& $-5424.66$ & $3.9713$ & $200.2$ \\ 

& & $10.0$ & $29954$ & $-6535.54$ & $5.1130$ & $1397.4$& $-5532.03$ & $4.0019$ & $588.0$ \\ 

\hline 
\end{tabular}
\end{footnotesize}
\\
\end{table}

\begin{table}[t!]
\caption{\small  Computational results for $\delta = 50\%$} 
\centering
\label{rdata-cov2}
\begin{footnotesize}
\begin{tabular}{|c c|c|c|c c c|c c c |}
\hline 
\multicolumn{2}{|c|}{Problem }  & \multicolumn{1}{c|}{} &
\multicolumn{1}{c|}{} & \multicolumn{3}{c|}{PPA} &
\multicolumn{3}{c|}{PD}
\\
\multicolumn{1}{|c}{$p$} & \multicolumn{1}{c|}{$|\Omega|$} &   
\multicolumn{1}{c|}{$\rho_{\bomega}$} & \multicolumn{1}{c|}{$r$} & 
\multicolumn{1}{c}{Likelihood}  & \multicolumn{1}{c}{Loss} & \multicolumn{1}{c|}{Time} &
\multicolumn{1}{c}{Likelihood}  & \multicolumn{1}{c}{Loss} & \multicolumn{1}{c|}{Time}
\\
\hline
500 & $37738$ & $0.01$ & $202226$ & $-947.33$ & $3.1774$ & $37.2$& $-935.11$ & $3.1134$ & $2.2$ \\ 

& & $0.10$ & $50118$ & $-1001.23$ & $3.3040$ & $41.8$& $-978.03$ & $3.1662$ & $4.7$ \\ 

& & $1.00$ & $11810$ & $-1052.09$ & $3.6779$ & $81.1$& $-101.80$ & $3.2889$ & $14.5$ \\ 

& & $10.0$ & $5032$ & $-1500.00$ & $5.0486$ & $71.1$& $-1041.64$ & $3.3966$ & $28.1$ \\ 

\hline 

1000 & $152512$ & $0.01$ & $816070$ & $-2225.875$ & $3.8864$ & $149.7$ & $-2201.98$ & $3.8126$ & $12.1$ \\ 

& & $0.10$ & $203686$ & $-2335.81$ & $4.0029$ & $131.0$ & $-2288.11$ & $3.8913$ & $17.2$ \\ 

& & $1.00$ & $46928$ & $-2400.81$ & $4.2945$ & $372.7$ & $-2349.02$ & $4.0085$ & $44.1$ \\ 

& & $10.0$ & $17370$ & $-3128.63$ & $5.5159$ & $265.2$& $-2390.09$ & $4.1138$ & $84.3$ \\ 

\hline 

1500 & $340656$ & $0.01$ & $1851266$ & $-3649.78$ & $4.2553$ & $361.2$& $-3616.72$ & $4.1787$ & $32.0$ \\ 

& & $0.10$ & $475146$ & $-3815.09$ & $4.3668$ & $303.4$& $-3743.19$ & $4.2725$ & $42.3$ \\ 

& & $1.00$ & $42902$ & $-3895.09$ & $4.6025$ & $1341.0$ & $-3874.68$ & $4.4823$ & $155.8$ \\ 

& & $10.0$ & $7430$ & $-4759.67$ & $5.6739$ & $881.2$ & $-4253.34$ & $4.6876$ & $468.6$ \\ 

\hline 

2000 & $605990$ & $0.01$ & $3301648$ & $-5149.12$ & $4.5763$ & $801.3$& $-5104.27$ & $4.5006$ & $61.7$ \\ 

& & $0.10$ & $893410$ & $-5371.26$ & $4.6851$ & $620.0$& $-5269.06$ & $4.5969$ & $82.4$ \\

& & $1.00$ & $153984$ & $-5456.54$ & $4.9033$ & $1426.0$& $-5406.89$ & $4.7614$ & $175.9$ \\ 

& & $10.0$ & $33456$ & $-6560.54$ & $5.9405$ & $1552.3$& $-5512.48$ & $4.7982$ & $565.5$ \\

\hline 
\end{tabular}
\end{footnotesize}
\\
\end{table}

\begin{table}[t!]
\caption{\small  Computational results for $\delta = 100\%$} 
\centering
\label{rdata-cov4}
\begin{footnotesize}
\begin{tabular}{|c c|c|c|c c c|c c c |}
\hline 
\multicolumn{2}{|c|}{Problem }  & \multicolumn{1}{c|}{} &
\multicolumn{1}{c|}{} & \multicolumn{3}{c|}{PPA} &
\multicolumn{3}{c|}{PD}
\\
\multicolumn{1}{|c}{$p$} & \multicolumn{1}{c|}{$|\Omega|$} &   
\multicolumn{1}{c|}{$\rho_{\bomega}$} & \multicolumn{1}{c|}{$r$} & 
\multicolumn{1}{c}{Likelihood}  & \multicolumn{1}{c}{Loss} & \multicolumn{1}{c|}{Time} &
\multicolumn{1}{c}{Likelihood}  & \multicolumn{1}{c}{Loss} & \multicolumn{1}{c|}{Time}
\\
\hline
500 & $0$ & $0.01$ & $238232$ & $-930.00$ & $3.5345$ & $36.0$& $-918.52$ & $3.4838$ & $1.3$ \\ 

& & $0.10$ & $57064$ & $-1000.78$ & $3.6826$ & $43.6$& $-973.06$ & $3.5313$ & $4.0$ \\ 

& & $1.00$ & $15474$ & $-1053.04$ & $4.0675$ & $76.1$& $-1006.95$ & $3.6500$ & $10.6$ \\ 

& & $10.0$ & $7448$ & $-1511.88$ & $5.4613$ & $51.4$& $-1023.82$ & $3.7319$ & $18.1$ \\ 

\hline 

1000 & $0$ & $0.01$ & $963400$ & $-2188.06$ & $4.1983$ & $156.3$& $-2161.58$ & $4.1383$ & $5.3$ \\ 

& & $0.10$ & $231424$ & $-2335.09$ & $4.3387$ & $122.4$& $-2277.90$ & $4.2045$ & $16.8$ \\ 

& & $1.00$ & $47528$ & $-2401.69$ & $4.6304$ & $329.6$& $-2349.74$ & $4.3449$ & $42.6$ \\ 

& & $10.0$ & $18156$ & $-3127.94$ & $5.8521$ & $244.1$& $-2388.22$ & $4.4466$ & $79.0$ \\ 

\hline 

1500 & $0$ & $0.01$ & $2181060$ & $-3585.21$ & $4.5878$ & $364.1$& $-3545.43$ & $4.5260$ & $12.3$ \\ 

& & $0.10$ & $551150$ & $-3806.07$ & $4.7234$ & $288.2$ & $-3717.25$ & $4.6059$ & $41.3$ \\ 

& & $1.00$ & $102512$ & $-3883.94$ & $4.9709$ & $912.8$ & $-3826.26$ & $4.7537$ & $93.5$ \\ 

& & $10.0$ & $31526$ & $-4821.26$ & $6.0886$ & $848.7$ & $-3898.50$ & $4.8824$ & $185.4$ \\ 

\hline 

2000 & $0$ & $0.01$ & $3892592$ & $-5075.44$ & $4.8867$ & $734.1$& $-5021.95$ & $4.8222$ & $23.8$ \\ 

& & $0.10$ & $1027584$ & $-5367.86$ & $5.0183$ & $590.6$& $-5246.45$ & $4.9138$ & $76.1$ \\ 

& & $1.00$ & $122394$ & $-5456.64$ & $5.2330$ & $1705.8$& $-5422.48$ & $5.1168$ & $197.8$ \\ 

& & $10.0$ & $25298$ & $-6531.08$ & $6.2571$ & $1803.4$& $-5636.74$ & $5.3492$ & $417.1$ \\ 

\hline 
\end{tabular}
\end{footnotesize}
\\
\end{table}

Our second experiment is similar to the one conducted in \cite{DaBaGh08,Lu10}. We intend to compare 
sparse recoverability of our PD method with PPA. To this aim, we specialize $p = 30$ and 
$(\bsigmat)^{-1} \in S^p_{++}$ to be the matrix with diagonal entries around one and a few randomly 
chosen, nonzero off-diagonal entries equal to $+1$ or $-1$. And the sample covariance matrix $\bsigma$ 
is then similarly generated as above. In addition, we set 
$\Omega = \{(i,j): (\bsigmat)^{-1}_{ij} = 0, \ |i-j| \geq 15\}$ and $\rho_{ij} = \rho_{\bomega}$ 
for all $(i,j) \in \bomega$, where $\rho_{\bomega}$ is the smallest number such that the approximate 
solution obtained by PPA shares the same number of nonzero off-diagonal entries as $(\bsigmat)^{-1}$. 
For problem \eqnok{Prob2}, we choose $r=\sum_{(i,j) \in \bomega} \|(\bsigmat)^{-1}_{ij}\|_0$ (i.e., 
the number of nonzero off-diagonal entries of $(\bsigmat)^{-1}$). PPA and the PD method are then 
applied to solve \eqnok{l1-cov} and \eqnok{Prob2} with the aforementioned $\rho_{ij}$ and $r$, 
respectively. In Figure \ref{fig:cov}, we plot the sparsity patterns of the original inverse 
covariance matrix $(\bsigmat)^{-1}$, the noisy inverse sample covariance matrix $\bsigma^{-1}$, and the approximate 
solutions to \eqnok{l1-cov} and \eqnok{Prob2} obtained by PPA and our PD method, respectively. We first observe 
that the sparsity of both solutions is the same as $(\bsigmat)^{-1}$. Moreover, the solution of 
our PD method completely recovers the sparsity patterns of $(\bsigmat)^{-1}$, but the solution of 
PPA misrecovers a few patterns. In addition, we present the log-likelihood and the normalized 
entropy loss of these solutions in Table \ref{recovery}. One can see that the solution of 
our PD method achieves much larger log-likelihood and smaller normalized entropy loss.    
   
\begin{figure}[t!]
\centering
\subfigure[\scriptsize True inverse $(\bsigmat)^{-1}$]{
\includegraphics[scale=0.3]{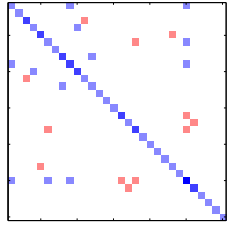}
\label{fig:cov-1}
}
\subfigure[\scriptsize Noisy inverse $\bsigma^{-1}$]{
\includegraphics[scale=0.3]{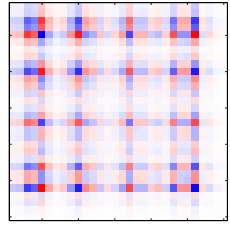}
\label{fig:cov-2}
}\\
\centering
\subfigure[\scriptsize Approximate solution of \eqnok{l1-cov}]{
\includegraphics[scale=0.3]{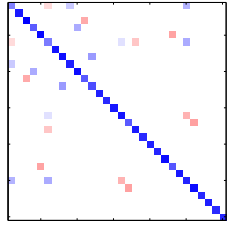}
\label{fig:cov-4}
}
\subfigure[\scriptsize Approximate solution of \eqnok{Prob2}]{
\includegraphics[scale=0.3]{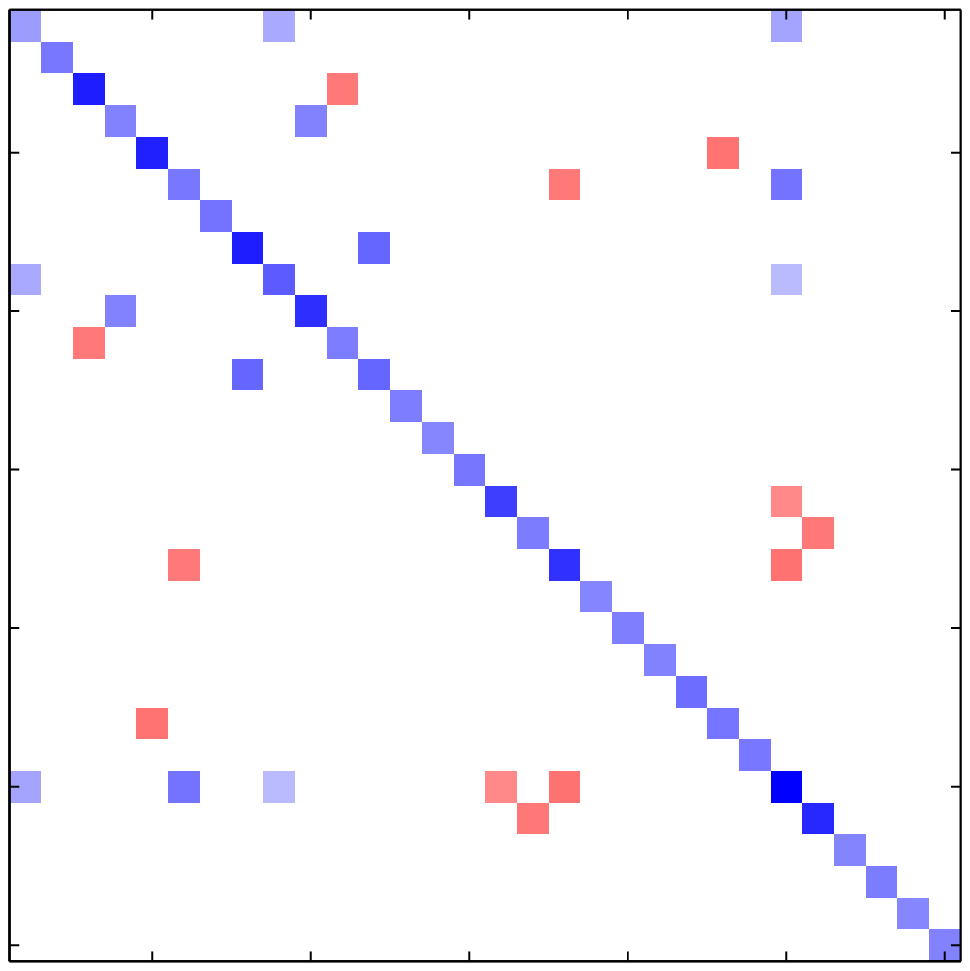}
\label{fig:cov-5}
}\\
\caption {Sparse recovery.}
\label{fig:cov}
\end{figure}

\begin{table}[t!]
\caption{\small  Numerical results for sparse recovery} 
\centering
\label{recovery}
\begin{footnotesize}
\begin{tabular}{|l|c|c|c|}
\hline 
\multicolumn{1}{|c|}{ }  & \multicolumn{1}{c|}{nnz} &
\multicolumn{1}{c|}{Likelihood} & \multicolumn{1}{c|}{Loss} 
\\
\hline
PPA  & $24$  & $-35.45$ & $0.178$  \\ 
PD   & $24$  & $-29.56$ & $0.008$\\
\hline
\end{tabular}
\end{footnotesize}
\\
\end{table}

In the third experiment, we aim to compare the performance of our PD method with the PPA on two gene expression data 
sets that have been widely used in the literature (see, for example, \cite{GoSlTaHu99,
PiHuDrHo04,YeBuRa05,Do09,LiTo10}). We first pre-process the data by the same procedure as described in \cite{LiTo10} to 
obtain a sample covariance matrix $\bsigma$, and set $\Omega = \emptyset$ and $\rho_{ij} = \rho_{\bomega}$ 
for some $\rho_{\bomega}>0$.  We apply PPA to solve problem \eqnok{l1-cov} with $\rho_{\bomega} = 0.01$, 
$0.05$, $0.1$, $0.5$, $0.7$ and $0.9$, respectively. For each $\rho_{\bomega}$, we choose $r$ to be the number of 
nonzero off-diagonal entries of the solution of PPA, which implies that the solution of the PD method when applied 
to \eqnok{Prob2} is at least as sparse as that of PPA. As the true covariance matrix $\bsigmat$ is unknown 
for these data sets, we now modify the normalized entropy loss defined above by replacing $\bsigmat$ by $\bsigma$.
The results of PPA and our PD method on these two data sets are presented in Table \ref{rdata-cov}. In detail, 
the name and dimension of each data set are given in the first three columns. The values of $\rho_{\bomega}$ and $r$ 
are listed in the fourth and fifth columns. The log-likelihood, the normalized entropy loss and the CPU time (in seconds) 
of PPA and the PD method are given in the last six columns, respectively. We can observe that our PD method is 
generally faster than PPA. Moreover, our PD method outperforms PPA in terms of log-likelihood and normalized 
entropy loss.

As a summary, the above experiments show that the quality of the approximate 
solution of \eqnok{Prob2} obtained by our PD method is generally better than 
that of \eqnok{l1-cov} found by PPA when the same sparsity is considered. 

\begin{table}[t!]
\caption{\small  Computational results on two real data sets} 
\centering
\label{rdata-cov}
\begin{footnotesize}
\begin{tabular}{|c|c|c|c|c|c c c|c c c |}
\hline 
\multicolumn{1}{|c|}{Data} & \multicolumn{1}{c|}{Genes } & 
\multicolumn{1}{c|}{Samples } & \multicolumn{1}{c|}{} &
\multicolumn{1}{c|}{} & \multicolumn{3}{c|}{PPA} &
\multicolumn{3}{c|}{PD}
\\

\multicolumn{1}{|c|}{} & \multicolumn{1}{c|}{$p$} & 
\multicolumn{1}{c|}{$n$} & \multicolumn{1}{c|}{$\rho_{\bomega}$} &
\multicolumn{1}{c|}{$r$} & \multicolumn{1}{c}{Likelihood} & \multicolumn{1}{c}{Loss} & \multicolumn{1}{c|}{Time} &
\multicolumn{1}{c}{Likelihood} & \multicolumn{1}{c}{Loss} & \multicolumn{1}{c|}{Time}
\\
\hline
 Lymph & $587$ & $148$ & $0.01$ & $144294$ & $790.12$ & $23.24$ & $101.5$ & $1035.24$ & $22.79$ & $38.0$ \\ 

& & &$0.05$ & $67474$ & $174.86$ & $24.35$ & $85.2$ & $716.97$ & $23.27$ & $31.5$ \\ 

& & &$0.10$ & $38504$ & $-47.03$ & $24.73$ & $66.7$ & $389.65$ & $23.85$ & $26.1$ \\ 

& & &$0.50$ & $4440$ & $-561.38$ & $25.52$ & $33.2$ & $-260.32$ & $24.91$ & $24.8$ \\ 

& & &$0.70$ & $940$ & $-642.05$ & $25.63$ & $26.9$ & $-511.70$ & $25.30$ & $22.0$ \\ 

& & &$0.90$ & $146$ & $-684.59$ & $25.70$ & $22.0$ & $-598.05$ & $25.51$ & $14.9$ \\ 

\hline 

 Leukemia & $1255$ & $72$ & $0.01$ & $249216$ & $3229.75$ & $28.25$ & $705.7$ & $3555.38$ & $28.12$ & $177.1$ \\ 

& & &$0.05$ & $169144$ & $1308.38$ & $29.85$ & $491.1$ & $2996.95$ & $28.45$ & $189.2$ \\ 

& & &$0.10$ & $107180$ & $505.02$ & $30.53$ & $501.4$ & $2531.62$ & $28.82$ & $202.8$ \\ 

& & &$0.50$ & $37914$ & $-931.59$ & $31.65$ & $345.9$ & $797.23$ & $30.16$ & $256.6$ \\ 

& & &$0.70$ & $4764$ & $-1367.22$ & $31.84$ & $125.7$ & $-1012.48$ & $31.48$ & $271.6$ \\ 

& & &$0.90$ & $24$ & $-1465.70$ & $31.90$ & $110.6$ & $-1301.99$ & $31.68$ & $187.8$ \\ 

\hline 

\end{tabular}
\end{footnotesize}
\\
\end{table}

\subsection{Compressed sensing} 
\label{sensing}

In this subsection, we apply the PD methods proposed in Section \ref{method} to solve the compressed sensing (CS) problem, which 
has important applications in signal processing (see, for example, \cite{ClMu73,TaBaMc79,LeFu81,SaSy86,ChDoSa98,Mi02,Tr06}). 

When the observation is noise free, the CS problem can be formulated as 
\beq \label{Prob3-1}
\min\limits_{x \in \Re^p} \{\|x\|_0: \ Ax=b\}, 
\eeq
where $A \in \Re^{n \times p}$ is a data matrix and $b \in \Re^n$ is an observation vector. One popular approach for 
finding an approximate solution to \eqref{Prob3-1} is to solve the following $l_1$ regularization problem: 
\beq \label{l1-cs}
\min\limits_{x \in \Re^p} \{\|x\|_1: \ Ax=b\}, 
\eeq
(see, for example, \cite{VaFr08,ChDoSa98}). Our aim below is to apply the PD method studied in Subsection 
\ref{sec:PD-o} to solve problem \eqref{Prob3-1} directly.

Clearly, problem \eqnok{Prob3-1} is in the form of \eqnok{l0-J2} and thus the PD method proposed in Subsection 
\ref{sec:PD-o} can be suitably applied to solve \eqnok{Prob3-1}. Also, one can observe that 
the main computation effort of the PD method when applied to \eqref{Prob3-1} lies in solving the 
subproblem arising in step 1a), which is in the form of 
\beq
\min_x \{\|x-c\|^2: Ax=b \} 
\label{Prob3-2} \\
\eeq
for some $c \in \Re^p$. It is well known that problem \eqnok{Prob3-2} has a closed-form solution given by
\[
x^* = c - A^T  (AA^T)^{-1}(Ac-b).
\]  
We now address the initialization and the termination criteria for the PD method. In particular, 
we choose $y^0_0 $ to be a feasible point of \eqnok{Prob3-1} with at most $n$ nonzero entries which 
can be obtained by executing the Matlab command $A\setminus b$. Also, we set the initial penalty 
parameter $\vrho_0=0.1$ and the parameter $\sigma = 10$. In addition, we use \eqnok{inner-term} and 
\[
\frac{\|x^k-y^k\|_\infty}{\max\{|p_{\vrho_k}(x^k,y^k)|,1\}} \le \eps_O
\]
as the inner and outer termination criteria, and set the associated accuracy parameters $\eps_I=10^{-5}$ 
and $\eps_O=10^{-6}$, respectively.  

We next conduct experiments to test the performance of our PD method for solving problem \eqnok{Prob3-1} 
on random data. We also compare the quality of the approximate solutions of \eqnok{Prob3-1} obtained by 
our PD method with that of \eqnok{l1-cs} found by a first-order solver SPGL1 \cite{VaFr08}. For the latter 
method, we use the default values for all parameters.

In the first experiment, given an integer $r \in [1,p]$, we randomly generate $100$ instances according to 
the standard Gaussian distribution. Each one consists of a sparse signal $u$ with cardinality $r$ 
and a data matrix $A \in \Re^{n \times p}$. Then we generate the corresponding observation vector 
$b$ by letting $b=Au$. In particular, we choose $n=1024$, $p=4096$. The values of $r$ range 
from $30$ to $300$ (see Table \ref{PD1}). We now try to recover $u$ by applying the PD method and SPGL1 to solve 
\eqref{Prob3-1} and \eqref{l1-cs}, respectively. To evaluate the solution quality of these methods, 
we adopt a similar criterion as described in \cite{ReFaPa07,CaRe08}. Given an approximate recovery $x^*$ for $u$, 
we define the mean squared error as
\[
\mbox{MSE} :=  \|x^*-u\|/p.
\]
We say $u$ is {\it successfully recovered} by $x^*$ if the cardinality of $x^*$ is the same as $u$ and 
moreover the corresponding MSE is less than $10^{-4}$. The computational results of both methods are 
presented in Table \ref{PD1}. In detail, the values of $r$ are given in the first column. The number 
of successfully recovered signals (NS) and the CPU time for both methods are reported in columns two 
to five, respectively. We observe that the recoverability of two methods is similar for the instances 
with relatively small $r$, but the PD method outperforms SPGL1 when $r$ becomes larger. We also see 
that the speed of both methods is comparable.

The second experiment is similar to the first one except that $A$ is randomly generated with orthonormal 
rows. The computational results of both methods are presented in Table \ref{PD2}.  We also observe that 
the PD method outperforms SPGL1 in terms of recoverability. 

\begin{table}[t!]
\caption{\small  Computational results for $A$ with non-orthonormal rows} 
\centering
\label{PD1}
\begin{footnotesize}
\begin{tabular}{|c ||c c |c c |}
\hline 
\multicolumn{1}{|c||}{} & \multicolumn{2}{c|}{SPGL1} &  \multicolumn{2}{c|}{PD} 
\\

\multicolumn{1}{|c||}{Cardinality} & 
\multicolumn{1}{c}{NS}  &  \multicolumn{1}{c|}{Time} & 
\multicolumn{1}{c}{NS}  & \multicolumn{1}{c|}{Time} 
\\
\hline
30  &    100  &   0.6 & 100   & 2.4 \\ 
60  &    100  &   1.0 & 100   & 2.8 \\ 
90  &    100  &   1.4 & 100  & 3.2 \\ 
120  &    100  & 2.0 & 100   & 3.5 \\ 
150  &    94 &   3.0 & 98   & 3.9 \\ 
180  &    93  &  4.8 & 97   & 4.7 \\ 
210  &    81  &  10.4 & 86  & 6.1 \\ 
240  &    22  & 23.6 & 68   & 12.4 \\
270  &    0  &   25.8 & 18   & 20.0 \\ 
300  &    0  &   28.2  &   0 &  22.6\\ 

\hline 

\end{tabular}
\end{footnotesize}
\\
\end{table}

\begin{table}[t!]
\caption{\small  Computational results for $A$ with orthonormal rows} 
\centering
\label{PD2}
\begin{footnotesize}
\begin{tabular}{|c ||c c |c c |}
\hline 
\multicolumn{1}{|c||}{} & \multicolumn{2}{c|}{SPGL1} &  \multicolumn{2}{c|}{PD} 
\\

\multicolumn{1}{|c||}{Cardinality} & 
\multicolumn{1}{c}{NS}  &  \multicolumn{1}{c|}{Time} & 
\multicolumn{1}{c}{NS}  & \multicolumn{1}{c|}{Time} 
\\
\hline
30  &    100  &   0.4 & 100   & 1.8 \\ 
60  &    100  &   0.6 & 100   & 2.1 \\ 
90  &    100  &   0.7 & 100  & 2.4 \\ 
120  &  100  & 1.0 & 100   & 2.8 \\ 
150  &    92 &  1.4 & 95   & 3.2 \\ 
180  &    91  &  2.1 & 95   & 4.0 \\ 
210  &    73 &  3.9 & 92  & 5.5 \\ 
240  &    29 & 9.0 & 61  & 12.1 \\
270  &    1  &   10.9  & 11   & 18.9 \\ 
300  &    0  &   11.3  &   1 &  19.1\\ 
\hline 

\end{tabular}
\end{footnotesize}
\end{table}

In the remainder of this subsection we consider the CS problem with noisy observation. In this 
case, the CS problem can be formulated as
\beq \label{Prob3}
\min\limits_{x \in \Re^p} \left\{\frac{1}{2} \|Ax-b\|^2: \ \|x\|_0 \leq r\right\}, 
\eeq
where $A \in \Re^{n \times p}$ is a data matrix, $b \in \Re^n$ is an observation vector, and 
$r \in [1,p]$ is some integer for controlling the sparsity of the solution. One popular approach 
for finding an approximate solution to \eqref{Prob3} is to solve the following $l_1$ regularization problem:
\beq
\min\limits_{x \in \Re^p} \frac{1}{2}\|Ax-b\|^2 + \lambda \|x\|_1,  \label{l1-cs1} 
\eeq
where $\lambda \geq 0$ is a regularization parameter (see, for example, \cite{FiNoWr07,HaYiZh07,KiKoLuBo07}). 
Our goal below is to apply the PD method studied in Subsection \ref{sec:PD-c} to solve \eqnok{Prob3} 
directly.

Clearly, problem \eqnok{Prob3} is in the form of \eqnok{l0-J1} and thus the PD method proposed in Subsection 
\ref{method} can be suitably applied to solve \eqnok{Prob3}. The main computation effort of the 
PD method when applied to \eqnok{Prob3} lies in solving the subproblem arising in step 1a), which is an 
unconstrained quadratic programming problem that can be solved by the conjugate gradient method. We now 
address the initialization and the termination criteria for the PD method. In particular, we randomly 
choose an initial point $y^0_0 \in \Re^p$ such that $\|y^0_0\|_0 \le r$. Also, we set the initial penalty 
parameter $\vrho_0=1$ and the parameter $\sigma = \sqrt{10}$. In addition, we use 
\[
\frac{|q_{\vrho_k}(x^k_{l+1},y^k_{l+1})- q_{\vrho_k}(x^k_{l},y^k_{l})|}{\max\{|q_{\vrho_k}(x^k_{l},y^k_{l})|,1\}}
\le \eps_I
\] 
and
\[
\frac{\|x^k-y^k\|_\infty}{\max\{|q_{\vrho_k}(x^k,y^k)|,1\}} \le \eps_O
\]
as the inner and outer termination criteria for the PD method, and set their associated accuracy parameters 
$\eps_I=10^{-2}$ and $\eps_O=10^{-3}$.

We next conduct numerical experiments to test the performance of our PD method for solving problem 
\eqnok{Prob3} on random data. We also compare the quality of the approximate solutions of \eqnok{Prob3} 
obtained by our PD method and the iterative hard-thresholding algorithm (IHT) \cite{BlDa08,BlDa09} with that 
of \eqnok{l1-cs1} found by a first-order solver GPSR \cite{FiNoWr07}. For IHT, we set $stopTol =10^{-6}$ and 
use the default values for all other parameters. And for GPSR, all the parameters are set as their default values.

We first randomly generate a data matrix $A \in \Re^{n \times p}$ and an observation vector $b\in \Re^n$ 
according to a standard Gaussian distribution.  Then we apply GPSR to problem \eqnok{l1-cs1} with a set 
of $p$ distinct $\lambda$'s so that the cardinality of the resulting approximate solution gradually 
increases from $1$ to $p$. Accordingly, we apply our PD method and IHT to problem \eqnok{Prob3} with 
$r = 1,\ldots,p$. It shall be mentioned that a warm-start strategy is applied to all three methods. That is, an 
approximate solution of problem \eqnok{Prob3} (resp., \eqnok{l1-cs}) for current $r$ (resp., $\lambda$) is used 
as the initial point for the PD method and IHT (resp., GPSR) when applied to the problem for next $r$ (resp., $\lambda$).  The 
average computational results of both methods over $100$ random instances with $(n,p) = (1024,4096)$ 
are plotted in Figure \ref{fig:PDCS1}. In detail, we plot the average residual $\|Ax-b\|$ against 
the cardinality in the left graph and the average accumulated CPU time \footnote{For a cardinality $r$, the 
corresponding accumulated CPU time is the total CPU time used to compute approximate solutions of problem 
\eqnok{Prob3} or \eqnok{l1-cs} with cardinality from $1$ to $r$.} (in seconds) against the cardinality in 
the right graph. We observe that the residuals of the approximate solutions of \eqnok{l1-cs1} obtained 
by our PD method and IHT are almost equal and substantially smaller than that of \eqnok{Prob3} found by GPSR 
when the same sparsity is considered. In addition, we can see that GPSR is faster than the other two methods.

We also conduct a similar experiment as above except that $A$ is randomly generated with orthonormal rows. 
The results are plotted in Figure \ref{fig:PDCS2}.  We observe that the PD method and IHT are generally slower 
than GPSR, but they have better solution quality than GPSR in terms of residuals.  

\begin{figure}[t!]
\centering
\subfigure[Residual vs. Cardinality]{
\includegraphics[scale=0.3]{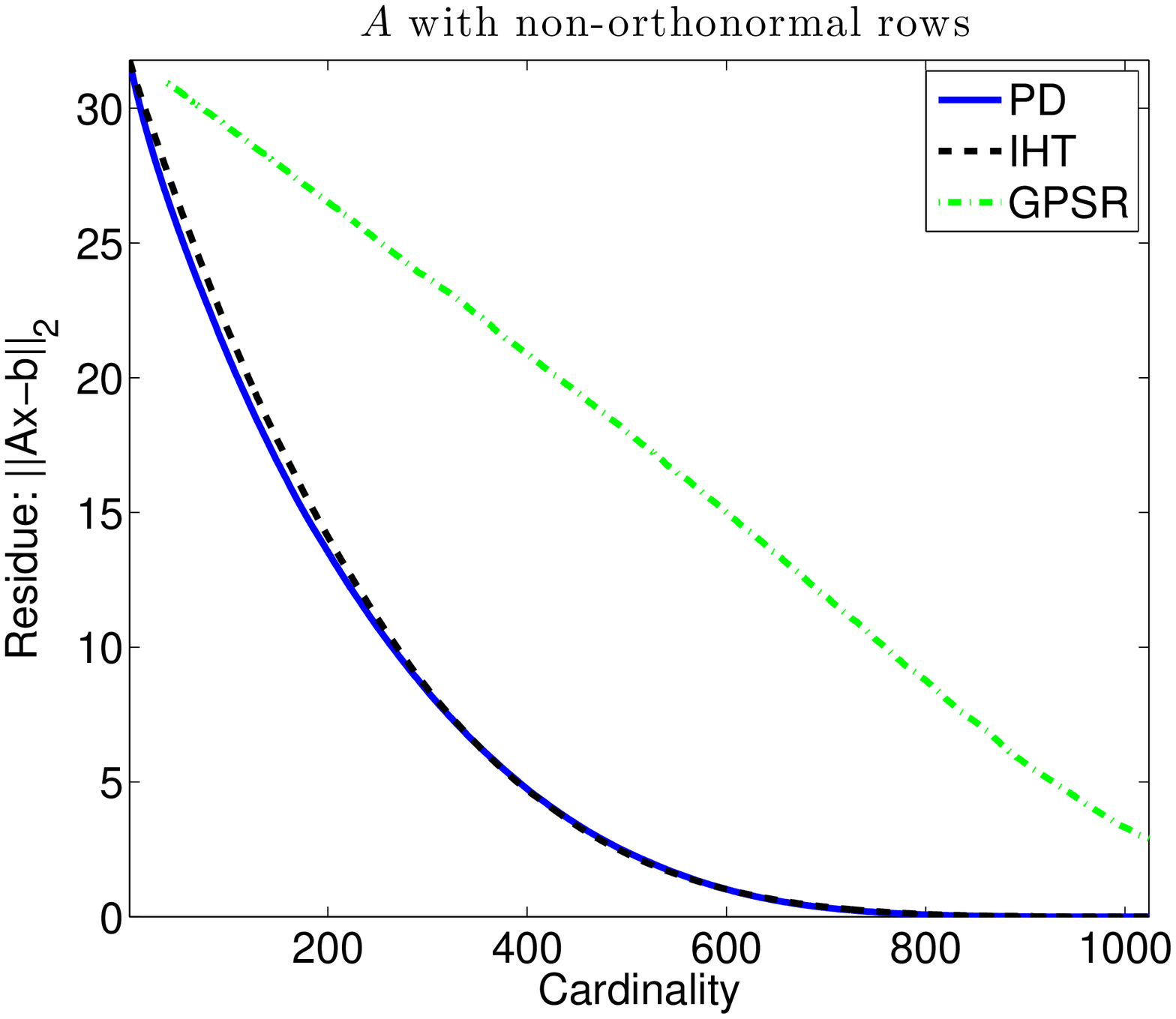}
\label{fig:PDCS1-1}
}
\subfigure[Time vs. Cardinality]{
\includegraphics[scale=0.3]{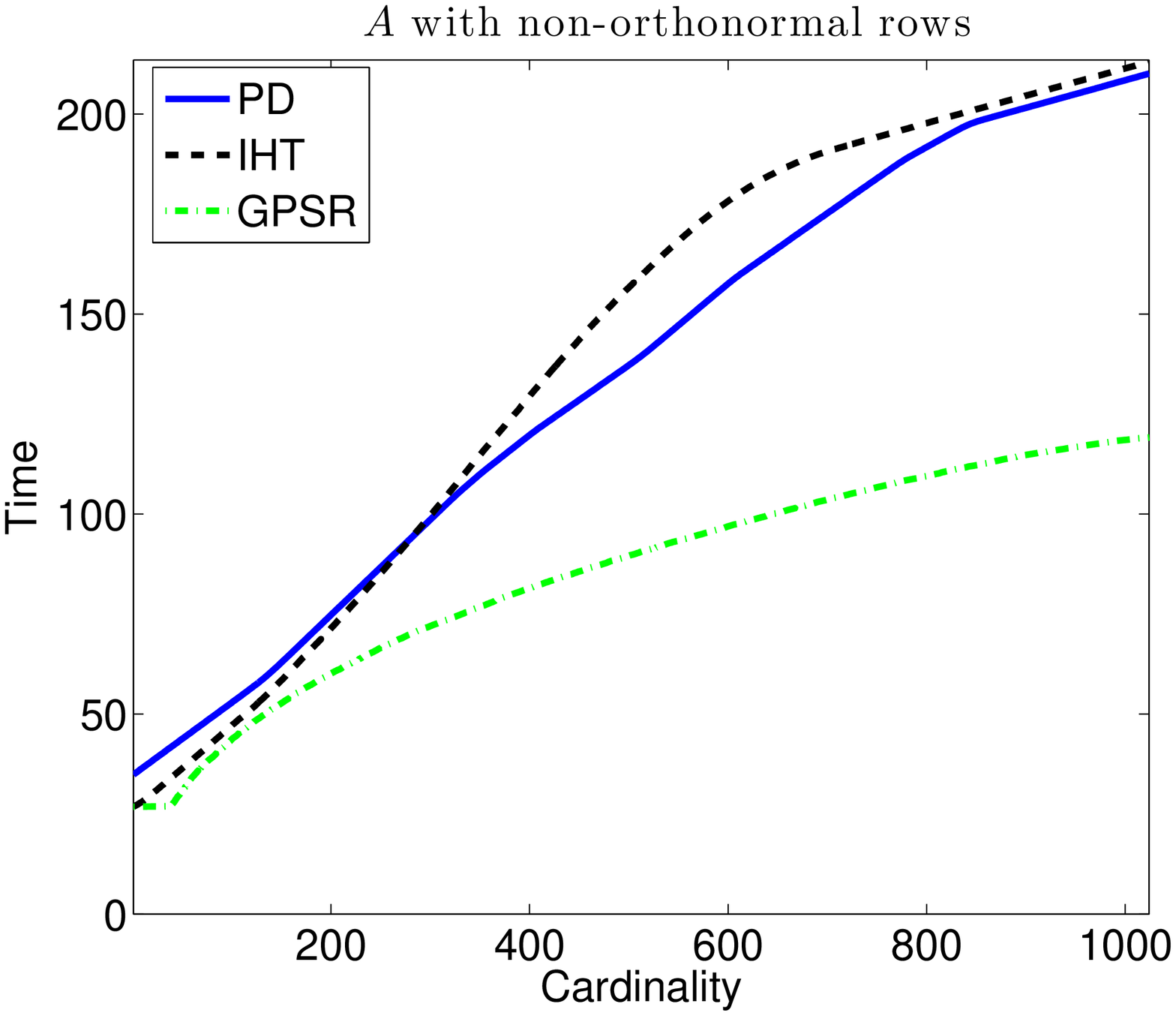}
\label{fig:PDCS1-2}
}\\
\caption {Trade-off curves.}
\label{fig:PDCS1}
\end{figure}

\begin{figure}[t!]
\centering
\subfigure[Residual vs. Cardinality]{
\includegraphics[scale=0.3]{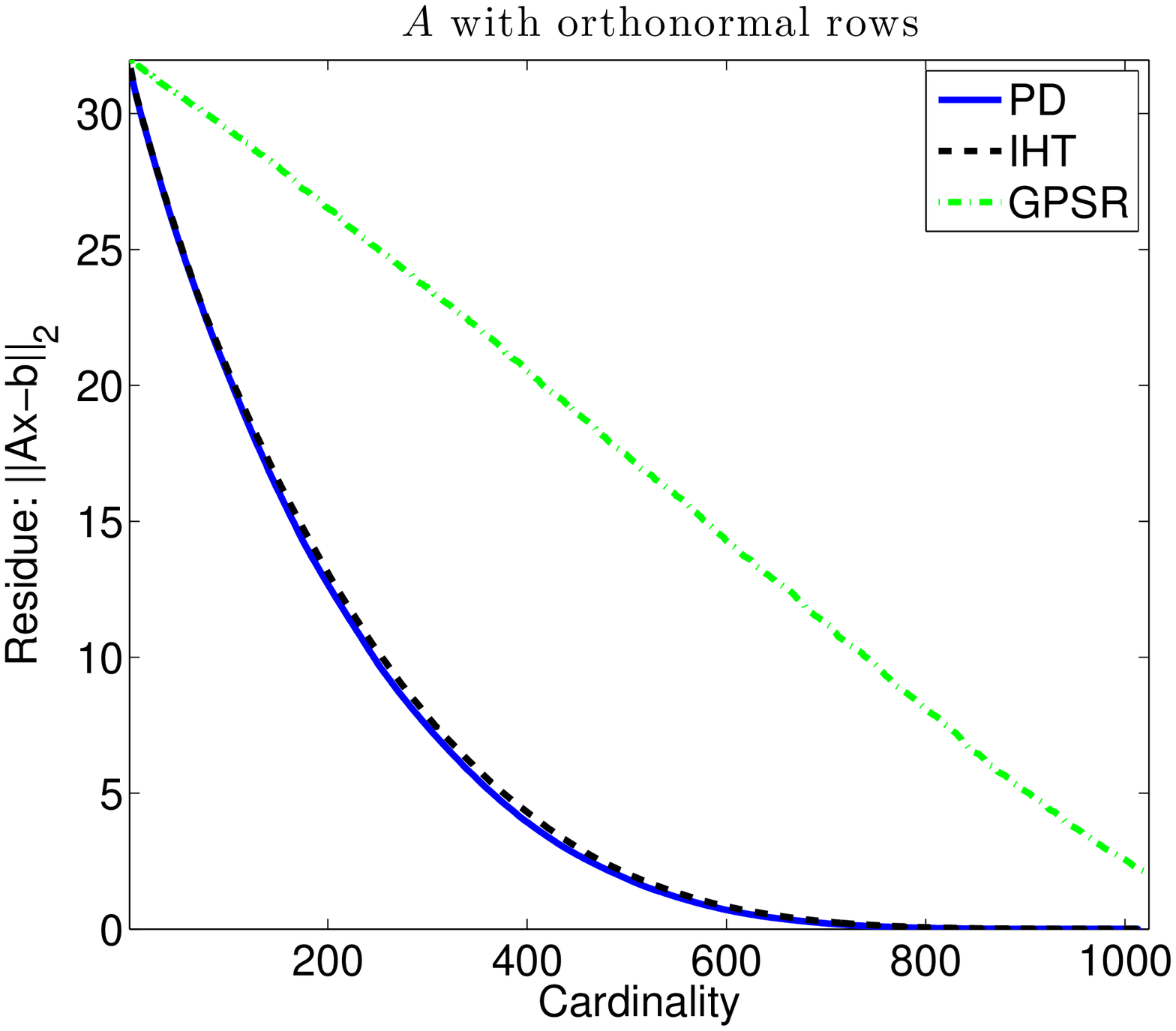}
\label{fig:PDCS2-1}
}
\subfigure[Time vs. Cardinality]{
\includegraphics[scale=0.3]{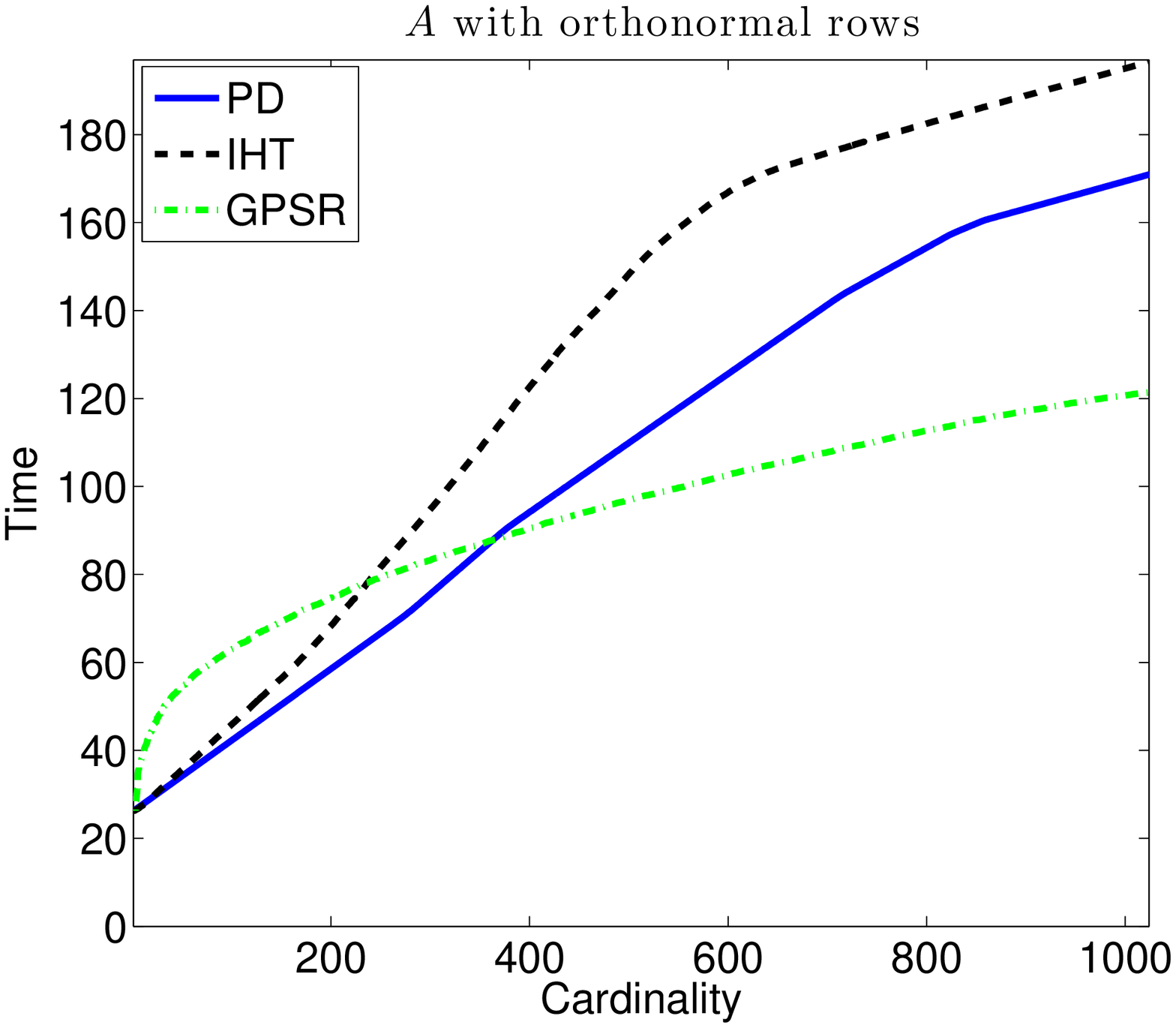}
\label{fig:PDCS2-2}
}\\
\caption {Trade-off curves.}
\label{fig:PDCS2}
\end{figure}

\section{Concluding remarks}
\label{conclude}

In this paper we propose penalty decomposition methods for general $l_0$ minimization problems 
in which each subproblem is solved by a block coordinate descend method. Under some suitable
assumptions, we establish that any accumulation point of the sequence generated by the PD methods 
satisfies the first-order optimality conditions of the problems. Furthermore, for the problems in 
which the $l_0$ part is the only nonconvex part, we show that such an accumulation point is a 
local minimizer of the problems. The computational results on compressed sensing, sparse logistic 
regression and sparse inverse covariance selection problems demonstrate that our methods generally 
outperform the existing methods in terms of solution quality and/or speed.

We shall remark that the augmented Lagrangian decomposition methods can be developed for 
solving $l_0$ minimization problems \eqref{l0-J1} and \eqref{l0-J2} simply by replacing the 
quadratic penalty functions in the PD methods by augmented Lagrangian functions. Nevertheless, 
as observed in our experiments, their practical performance is generally worse than the PD methods. 
    
\section*{Appendix}
In this appendix we provide an example to demonstrate that the $l_p$-norm relaxation approaches for 
$p \in (0,1]$ may fail to recover the sparse solution.
       
Let $p\in (0,1]$ be arbitrarily chosen. Given any $b^1$, $b^2 \in \Re^n$, let $b = b^1 + b^2$, 
$\alpha=\|(b^1;b^2)\|_p$ and $A = [b^1, \ b^2, \ \alpha I_n, \ \alpha I_n]$, where $I_n$ denotes 
the $n \times n$ identity matrix and $\|x\|_p = (\sum^n_{i=1} |x_i|^p)^{1/p}$ for all $x\in\Re^n$. 
Consider the linear system $Ax=b$. It is easy to observe that this system has the sparse solution 
$x^s=(1,1,0,\ldots,0)^T$. However, $x^s$ cannot be recovered by solving the $l_p$-``norm'' regularization 
problem:
\[
f^* = \min\limits_x\left\{f(x) := \frac12\|Ax-b\|^2 + \nu \|x\|_p\right\}
\] 
for any $\nu >0$. Indeed, let $\bar x = (0,0,b^1/\alpha,b^2/\alpha)^T$. Then, we have $f(x^s) = 2^{1/p}\nu$ and 
$f(\bar x)=\nu$, which implies that $f(x^s) > f(\bar x) \ge  f^*$. Thus, $x^s$ cannot be an optimal solution of the 
above problem for any $\nu >0$. Moreover, the relative error between $f(x^s)$ and $f^*$ is fairly large since
\[
(f(x^s) - f^*)/f^*  \ \ge \ (f(x^s) -f(\bar x))/f(\bar x) \ = \ 2^{1/p} - 1 \ \ge \ 1.   
\] 
Therefore, the true sparse solution $x^s$ may not even be a ``good'' approximate solution to the $l_p$-``norm'' 
regularization problem.

\end{document}